\documentclass[lettersize,journal]{IEEEtran}
\IEEEoverridecommandlockouts  
\usepackage{times}
\usepackage{graphicx}
\usepackage{amsmath}
\usepackage{amssymb}
\usepackage[footnotesize]{caption}
\usepackage{mathtools}
\usepackage[normalem]{ulem}
\usepackage{algorithm}
\usepackage{gensymb}
\usepackage{algorithmic}
\usepackage{float,bm}
\usepackage{hyperref}
\usepackage[style=ieee, backend=biber]{biblatex}
\usepackage{outlines}

\usepackage{enumitem}
\setenumerate[1]{label=\Roman*.}
\setenumerate[2]{label=\Alph*.}
\setenumerate[3]{label=\arabic*.}
\setenumerate[4]{label=\roman*.}

\usepackage{hyperref}

\def\xHyphenate#1#2\wholeString {\if#1$%
    \else\transform{#1}%
    \takeTheRest#2\ofTheString\fi}
\def\takeTheRest#1\ofTheString\fi
{\fi \xHyphenate#1\wholeString}
\def\transform#1{\url{#1}\hskip 0pt plus 1pt}

\hypersetup{
    colorlinks=true,
    linkcolor=blue,
    filecolor=magenta,      
    urlcolor=cyan,
    pdftitle={REDEFINED},
    pdfpagemode=FullScreen,
    }
    
\usepackage{multicol}

\usepackage{amsmath} 

\usepackage{amsthm}
\usepackage{float}
\usepackage{multirow}
\usepackage{caption}
\usepackage{subcaption}
\usepackage{makecell}
\usepackage{graphicx, caption, subcaption}

\providecommand{\R}{\ensuremath \mathbb{R}}
\providecommand{\N}{\ensuremath \mathbb{N}}

\newcommand{\nut}{n_u}
\newcommand{\nzt}{n_d}

\providecommand{\Rnut}{\ensuremath \mathbb{R}^{\nut}}
\providecommand{\Rnzt}{\ensuremath \mathbb{R}^{\nzt}}
\newtheorem{defn}{Definition}

\newtheorem{lem}[defn]{Lemma}

\newtheorem{assum}[defn]{Assumption}

\newtheorem{thm}[defn]{Theorem}
\newtheorem{cor}[defn]{Corollary}

\newcommand{\defemph}[1]{\emph{#1}}
\newcommand{\ts}[1]{\textsuperscript{#1}}
\newcommand{\relu}{\texttt{ReLU}}

\providecommand{\methodname}{\text{REDEFINED}}

\providecommand{\Int}{\texttt{int}}

\providecommand{\Zvel}{\mathcal Z^{\text{vel}}}

\providecommand{\zvel}{z^{\text{vel}}}
\providecommand{\zpos}{z^{\text{pos}}}

\providecommand{\diag}{\texttt{diag}}

\providecommand{\rot}{\texttt{rot}}

\providecommand{\cost}{\texttt{cost}}
\providecommand{\Oego}{\mathcal O^\text{ego}}
\providecommand{\vegomax}{\nu^\text{ego}}
\providecommand{\vobsmax}{\nu^\text{obs}}
\providecommand{\tz}{t_0}
\providecommand{\tplan}{t_\text{plan}}
\providecommand{\tnb}{t_\text{m}}
\providecommand{\tf}{t_\text{f}}
\providecommand{\tm}{t_\text{m}}

\providecommand{\opt}{\texttt{(Opt)} }
\providecommand{\bopt}{\texttt{(Batched Opt)}}
\providecommand{\nnopt}{\texttt{(NN-Opt)}}

\providecommand{\thopt}{(\texttt{Opt}_{\texttt{TH}})}

\providecommand{\optref}{\hyperref[eq:optcost]{\opt}}

\providecommand{\nnoptref}{\hyperref[eq:nnoptcost]{\nnopt}}

\providecommand{\onlineopt}{\hyperref[eq:onlineopt]{\texttt{OnlineOpt}}}
\providecommand{\senseobs}{\hyperref[eq:senseobs]{\texttt{SenseObstacles}}}
\providecommand{\statepred}{\hyperref[eq:stateprediction]{\texttt{StatePrediction}}}

\providecommand{\PP}{\mathcal{P}}

\providecommand{\W}{\mathcal{W}}

\providecommand{\K}{\mathcal{K}}
\providecommand{\Z}{\mathcal{Z}}

\renewcommand{\P}{\mathcal{P}}

\providecommand{\Z}{\mathcal{Z}}

\providecommand{\T}{\mathcal{T}}

\providecommand{\I}{\mathcal{I}}
\providecommand{\J}{\mathcal{J}}
\providecommand{\E}{\mathcal{E}}

\providecommand{\FOzj}{\xi_j}

\providecommand{\W}{\mathcal{W}}

\newcommand{\conv}{co}

\providecommand{\T}{\ensuremath T}

\newcommand{\zonocg}[2]{ \text{\textless} #1,\; #2 \text{\textgreater}}

\newcommand{\stkout}[1]{{\color{Plum}\ifmmode\text{\sout{\ensuremath{#1}}}\else\sout{#1}\fi}}

\newcommand{\dist}{d}
\newcommand{\sdf}{s}
\newcommand{\rdf}{r}

\newcommand{\ardf}{\Tilde{r}}

\newcommand{\cz}{c_z}

\newcommand{\coi}{c_{o,i}}
\newcommand{\Gz}{G_z}

\newcommand{\Goi}{G_{o,i}}

\newcommand{\Oi}{\mathcal{O}_i}
\newcommand{\Oiz}{\mathcal{O}_{i,z}}

\newcommand{\Obs}{\mathcal{O}}

\newcommand{\Obskt}{\mathcal{O}_k(t)}

\newcommand{\Obszjk}{\mathcal{O}_{j,k}}

\providecommand{\heading}{h(t)}
\providecommand{\sx}{{w_{\text{x}}}}
\providecommand{\sy}{{w_{\text{y}}}}

\providecommand{\segli}{\varphi_{l,i}}

\providecommand{\sB}{ \mathcal B}

\begin{document}

\title{Reachability-based Trajectory Design via Exact Formulation of Implicit Neural Signed Distance Functions}
\author{Jonathan Michaux*, Qingyi Chen*, Challen Enninful Adu, Jinsun Liu, and Ram Vasudevan
\thanks{Jonathan Michaux, Qingyi Chen, Challen Enninful Adu, Jinsun Liu, and Ram Vasudevan are with the Department of Robotics, University of Michigan, Ann Arbor, MI 48109. \texttt{\{jmichaux, chenqy, enninful, jinsunl, ramv\}@umich.edu}.}
\thanks{$*$These two authors contributed equally to this work.}
}

\maketitle
\IEEEpeerreviewmaketitle

\begin{abstract}
Generating receding-horizon motion trajectories for autonomous vehicles in real-time while also providing safety guarantees is challenging.
This is because a future trajectory needs to be planned before the previously computed trajectory is completely executed.
This becomes even more difficult if the trajectory is required to satisfy continuous-time collision-avoidance constraints while accounting for a large number of obstacles.
To address these challenges, this paper proposes a novel real-time, receding-horizon motion planning algorithm named \textbf{Re}achability-based trajectory \textbf{D}esign via \textbf{E}xact \textbf{F}ormulation of \textbf{I}mplicit \textbf{NE}ural signed \textbf{D}istance functions (\methodname{}).
\methodname{} first applies offline reachability analysis to compute zonotope-based reachable sets that overapproximate the motion of the ego vehicle.
During online planning, \methodname{} leverages zonotope arithmetic to construct a neural implicit representation that computes the exact signed distance between a parameterized swept volume of the ego vehicle and obstacle vehicles.
\methodname{} then implements a novel, real-time optimization framework that utilizes the neural network to construct a collision avoidance constraint.
\methodname{} is compared to a variety of state-of-the-art techniques and is demonstrated to successfully enable the vehicle to safely navigate through complex environments.
Code, data, and video demonstrations can be found at \url{https://roahmlab.github.io/redefined/}.
\end{abstract}

\section{Introduction}\label{sec:intro}
\begin{figure}[t]
    \includegraphics[width=\columnwidth]{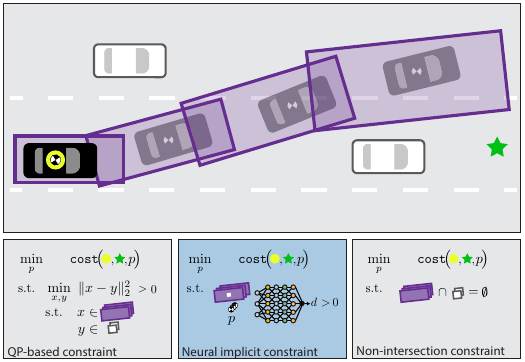}
    \caption{
    \methodname{} constructs safe trajectories in a receding-horizon fashion from an initial state (yellow circle) to a goal state (green star) while avoiding obstacles (white cars).
    \methodname{} first performs offline reachability analysis using a closed-loop full-order vehicle dynamics model to construct control-parameterized, zonotope reachable sets (purple boxes) that over-approximate all possible states of the vehicle over a finite planning horizon. 
     During online planning, \methodname{} computes a parameterized controller by solving an optimization problem that selects subsets of pre-computed zonotope reachable sets.
     Typically this can be done by either requiring that the distance between the obstacles and reachable set are greater than zero (QP-based formulation) or by incorporating a constraint that ensures that the two sets do not intersect with one another (non-intersection constraint); however, solving these formulations of the trajectory design problem can be computationally prohibitive.
    Instead, \methodname{} uses a novel neural implicit representation to compute the exact signed distance between the vehicle's zonotope reachable sets and obstacles.
    By ensuring the signed distance between the reachable sets and obstacles is greater than zero, \methodname{} is able to guarantee the ego vehicle does not collide with any obstacles while executing the trajectory.}
    \label{fig:high_level}
    \vspace*{-0.25cm}
\end{figure}
Autonomous vehicles have the potential to significantly reduce the number of car accidents and collisions.
To realize these possibilities, autonomous vehicles must be capable of operating safely in unknown environments and predicting the behavior of dynamic obstacles with limited sensing horizons.
Because new sensor information is received while the autonomous vehicle is moving, it is necessary to plan trajectories in a receding-horizon fashion in which each new trajectory is planned while the vehicle executes the trajectory computed in the previous planning iteration.
To ensure safe autonomous driving in the real-world, a motion planning framework must satisfy three properties.
First, the planner should guarantee that every computed trajectory satisfies the physical constraints of the vehicle.
Second, the planner should operate in real-time.
Because each planned trajectory relies upon predictions of the environment that become less accurate as time increases, short planned trajectories and fast planning times ensure the vehicle can quickly react to environmental changes. 
Finally, the planner should verify that any computed trajectory is collision-free for its entire duration when realized by the vehicle.
This paper proposes \methodname{}, a receding-horizon trajectory planning framework that uses a novel neural network representation of the exact signed distance function to guarantee vehicle safety.
\begin{figure*}[t]
    \centering
    \includegraphics[width=0.99\textwidth]{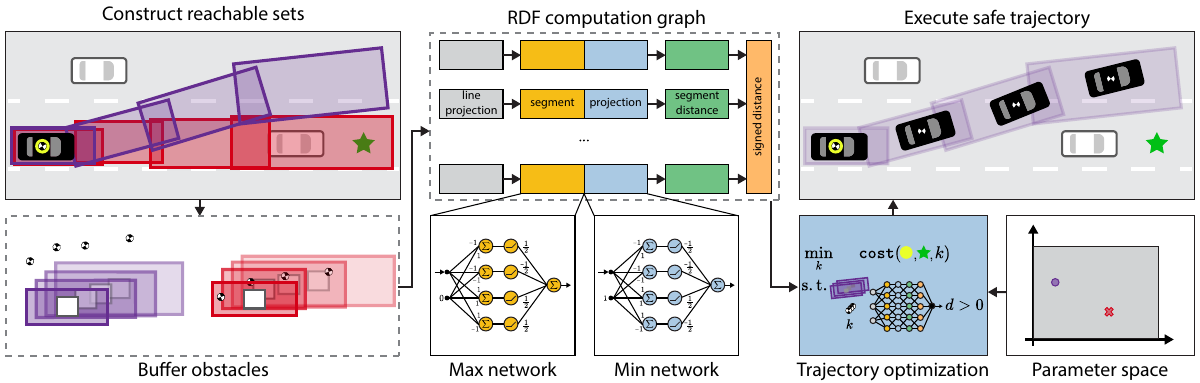}
    \caption{Overview of \methodname{}. During each planning iteration, a family of reachable sets are constructed that correspond to different maneuvers.
    In this figure, subsets of purple zonotope reachable sets corresponding to the control parameter shown in purple ensures a collision-free path while the other control parameter shown magenta might lead to collisions with white obstacles.}
    \label{fig:overview}
    \vspace*{-0.2cm}
\end{figure*}

Ensuring safe motion planning for autonomous vehicles in real-time while considering vehicle dynamics requires accurate predictions of vehicle behavior within the planning time horizon. 
However, the nonlinear vehicle dynamics makes closed-form solutions for vehicle trajectories computationally infeasible, necessitating the use of approximations.
Sampling-based methods are commonly employed to explore the environment and generate motion plans. 
These methods discretize the system dynamic model or state space to find an optimal path toward the goal location, based on a user-specified cost function \cite[]{janson2015fast}. 
However, modeling vehicle dynamics during real-time planning using sampling-based methods introduces challenges, as online numerical integration is needed, and obstacles must be buffered to compensate for numerical integration errors \cite[]{elbanhawi2014sampling, lavalle2006planning, kuwata2009real}. 
Unfortunately, achieving both dynamic realizability and collision-free trajectories often demands fine time discretization, leading to an undesirable trade-off between safety and real-time operation.
Another approach commonly used for motion planning in autonomous vehicles is Nonlinear Model Predictive Control (NMPC). 
NMPC employs time discretization to approximate solutions to the vehicle dynamics. 
This approximation is then incorporated into an optimization program to compute control inputs that ensure dynamic realizability while avoiding obstacles \cite[]{howard2007optimal, falcone2008low, wurts2018collision}.
Similar to sampling-based methods, NMPC also faces challenges in balancing safety and real-time performance.

Recent research in motion planning have introduced reachability-based planning algorithms as a safe alternative to traditional sampling-based methods or NMPC, mitigating the trade-off between safety and real-time operation. 
Initially, reachable sets were constructed to overapproximate the behavior of a maneuver and were used to verify that a maneuver could be applied safely \cite{althoff2014online}.
Note, this approach was focused on online verification of a trajectory computed by an arbitrary trajectory planner; as a result, the trajectory planner was not required to compute a a safe plan. 
This online verification layer has been extended and applied to urban traffic situations by trying to address the problem of trajectory synthesis by using discrete time approximations to the vehicle dynamics \cite[]{manzinger2020using, schafer2021computation}.


However, computing the FRS is particularly challenging for nonlinear or high-dimensional systems. 
To address this, researchers have focused on simplifying dynamics and pre-specifying maneuvers. 
Reachability-based Trajectory Design (RTD) \cite[]{kousik2020bridging}, considers a continuum of trajectories and employs SOS programming to represent the FRS as a polynomial level set. 
This approach ensures strict vehicle safety guarantees while maintaining online computational efficiency, but it can still face challenges with high-dimensional systems. 
RTD also relies on a simplified, low-dimensional nonlinear model to bound the full-order vehicle behavior, potentially leading to conservative safety margins and impractical assumptions.
To address these limitations, another state-of-the-art reachability based method, REFINE \cite[]{liu2022refine}, leverages a parameterized feedback linearizing controller coupled with a zonotope-based FRS to ensure safety during motion planning.
Importantly, this approach is able to use a full order vehicle model.
As a result, REFINE grants strong guarantees of safety, while maintaining online computational efficiency.
However, as we illustrate in this paper, REFINE's planning time does not scale well with increasing number of obstacles which becomes a challenge in dense traffic environments.


As illustrated in Fig. \ref{fig:high_level}, this paper proposes a real-time, receding-horizon motion planning algorithm  named \textbf{Re}achability-based Trajectory \textbf{D}esign via \textbf{E}xact \textbf{F}ormulation of \textbf{I}mplicit \textbf{Ne}ural Signed \textbf{D}istance Functions (\methodname{}).
This paper's contributions are three-fold: First, A neural implicit representation called \methodname{} that computes the exact signed distance between a parameterized swept volume trajectory and obstacles; second, a real-time optimization framework that utilizes \methodname{} to construct a collision avoidance constraint; and third a demonstration that \methodname{} outperforms similar state-of-the-art methods on a set of challenging motion planning tasks.

The remainder of this manuscript is organized as follows:
Section \ref{sec:prelim} summarizes the set representations used throughout the paper and gives a brief overview of signed distance functions;
Section \ref{sec:vehicle_modeling} describes the vehicle model and the environment;
Section \ref{sec:overview} describes the formulation of the safe motion planning problem for the autonomous vehicle;
Section \ref{sec:nn} describes how to build an exact representation of a signed distance using ReLU networks and use it within a safe motion planning framework;
Sections \ref{sec:experimental_setup} and \ref{sec:results} summarize the evaluation of the proposed method on a variety of different example problems.

\section{Preliminaries}\label{sec:prelim}
This section establishes the notation used in the paper.

\subsection{Notation}
Sets and subspaces are typeset using capital letters.
Subscripts are primarily used as an index or to describe a particular coordinate of a vector.
Let $\R$ and $\N$ denote the spaces of real numbers and natural numbers, respectively.
The Minkowski sum between two sets $\mathcal A$ and $\mathcal A'$ is $\mathcal A\oplus \mathcal A' = \{a+a'\mid a\in \mathcal A, ~a'\in \mathcal A'\}$.
Given a set $\mathcal A$, denote its power set as $P(\mathcal A)$.
Given vectors $\alpha,\beta\in\R^n$, let $[\alpha]_i$ denote the $i$-th element of $\alpha$, 
let $\diag(\alpha)$ denote the diagonal matrix with $\alpha$ on the diagonal, and  let $\Int(\alpha,\beta)$ denote the $n$-dimensional box defined as $\{\gamma \in\R^n\mid [\alpha]_i\leq[\gamma]_i\leq[\beta]_i,~\forall i=1,\ldots,n\}$.
Given $\alpha \in \R^n$ and $\epsilon > 0 $, let $B(\alpha, \epsilon)$ denote the $n$-dimensional closed ball with center $\alpha$ and radius $\epsilon$ under the Euclidean norm.
Let $\rot(a)$ denote the 2-dimensional rotation matrix $\begin{bmatrix}
        \cos(a) & -\sin(a) \\ \sin(a) & \cos(a)
    \end{bmatrix}$ for arbitrary $a\in\R$.
Define the Rectified Linear Unit (ReLU) as $\relu(z) = \max\{0,z\}$ where $z \in \R^n$, the $\max$ is taken element-wise, and $0$ is the zero vector in $\R^n$.
Note all norms are assumed to the be the 2-norm unless stated otherwise. 

\subsection{Set-based Representations and Operations}
\label{subsec:set_representations}
Given a set $\Omega \subset \mathbb{R}^{n}$, let $\partial \Omega \subset \mathbb{R}^{n}$ be its boundary, $\Omega^{c} \subset \mathbb{R}^{n}$ denote its complement, and $\conv(\Omega)$ denote its convex hull.
Next, we introduce a subclass of polytopes, called zonotopes:
\begin{defn}
\label{def: zonotope}
A \emph{zonotope} $\Z$ is a subset of $\R^n$ defined as
\begin{equation}
    \mathcal Z = \left\{x\in\R^n\mid x= c+\sum_{k=1}^m \beta_k g_k, \quad \beta_k \in [-1,1] \right\}
\end{equation}
with \emph{center} $c\in\R^n$ and $m$ \emph{generators} $g_1,\ldots,g_m\in\R^n$.
For convenience, we denote $\mathcal Z$ by $\zonocg{c}{G}$ where $G = [g_1, g_2, \ldots, g_\ell ]\in\R^{n\times\ell}$.
\end{defn}
By definition the Minkowski sum of two arbitrary zonotopes  $\Z_1 = \zonocg{c_1}{G_1}$ and $\Z_2=\zonocg{c_2}{G_2}$ is still a zonotope as $\Z_1\oplus\Z_2 = \zonocg{c_1+c_2}{[G_1,G_2]}$.
Note that one can define the multiplication of a matrix $A$ of appropriate size with a zonotope $\Z=\zonocg{c}{G}$ as
\begin{equation}
\label{eq: zono-matrix mult}
   A \Z = \left\{x\in\R^n\mid x= A c+\sum_{k=1}^m \beta_k A g_k, ~ \beta_k \in [-1,1] \right\}.
\end{equation}
Note that $A \Z$ is equal to the zonotope $\zonocg{A c}{A G}$. 

\subsection{Overview of Signed Distance Functions}
\label{subsec:def_sdf}
\begin{defn}
Given a set $\Omega \subset \mathbb{R}^{n}$, the \defemph{distance function} associated with $\Omega$ is defined by
    \begin{equation}
       \dist(x;\Omega) = \min_{y \in \partial\Omega} \|x - y\|.
    \end{equation}
The \defemph{distance between two sets} $\Omega_1, \Omega_2 \subset \mathbb{R}^{n}$ is defined by
    \begin{equation}
       \dist(\Omega_1,\Omega_2) = \min_{\substack{x \in \partial\Omega_1 \\ y \in \partial\Omega_2}} \|x - y\|.
    \end{equation}
\end{defn}
\noindent Note that in the instance that the arguments to $d$ are convex polytopes, such as zonotopes, one can compute the distance by applying convex optimization \cite[p.398]{boyd2004convex}.
In addition, note that $d$ is zero for sets that have non-trivial intersection. 
As a result, this distance function provides limited information about how much a pair of sets are intersecting with one another.
To address this limitation, consider the following definition:
\begin{defn} Given a subset $\Omega$ of $\mathbb{R}^{n}$, the \defemph{signed distance function} $\sdf$ between a point is a map $\sdf: \mathbb{R}^{n} \to \mathbb{R}$ defined as
    \begin{equation}
    \sdf(x; \Omega) = 
        \begin{cases}
          \dist(x; \partial\Omega)  & \text{if } x \in \Omega^{c} \\
          -\dist(x;\partial\Omega) & \text{if } x \in \Omega.
        \end{cases}
\end{equation}
The \defemph{signed distance between two sets} $\Omega_1, \Omega_2 \subset \mathbb{R}^{n}$ is 
    \begin{equation}
    \sdf(\Omega_1,\Omega_2) = 
        \begin{cases}
           \min_{\substack{x \in \partial\Omega_1 \\ y \in \partial\Omega_2}} \|x - y\| & \text{if } \Omega_1 \cap \Omega_2 = \emptyset  \\
          \min_{\substack{x \in \partial\Omega_1 \\ y \in \partial\Omega_2}} -\|x - y\| & \text{otherwise}.
        \end{cases}
\end{equation}
\end{defn}


\section{Vehicle Modeling and Environment}
\label{sec:vehicle_modeling}

This section discusses the ego vehicle's model, parameterized trajectories, and environment.

\subsection{Ego Vehicle Dynamic Model}
\label{subsec: dynamics}
Throughout this work, the dynamics of the ego vehicle satisfy a nonlinear differential equation:
\begin{equation} \label{eq:dynamics}
\dot{z}(t) = f(z(t),u(t)),
\end{equation}
where $z(t) \in \Rnzt$ is the state of the ego vehicle at time $t$ and $u(t) \in \Rnut$ is the control input at time $t$. 
For convenience, we assume that the first two components of $z(t)$, which we denote by $\sx(t)$ and $\sy(t)$, correspond to the position of the vehicle in a two dimensional world space that we denote by $\W\subset\R^2$. 
We assume that the third component of $z(t)$, which we denote by $\heading$, corresponds to the ego vehicle's heading in the same world coordinate frame. 
Examples of dynamic models that satisfy this requirement include the Dubin's or bicycle model.

\subsection{Trajectory Parameterization}
\label{subsec: trajectory param}
In this work, each trajectory is specified over a compact time interval of a fixed duration $\tf$. 
Because $\methodname{}$ performs receding-horizon planning, we make the following assumption about the time available to construct a new plan: 

\begin{assum} \label{assum:tplan}
During each planning iteration starting from time $\tz$, the ego vehicle has $\tplan$ seconds to find a control input that is applied during the time interval $[\tz+\tplan, \tz+\tplan+\tf]$, where $\tf \geq 0$ is some user-specified constant. 
In addition, the vehicle state at time $\tz+\tplan$ is known at time $\tz$. 
\end{assum}

\noindent This assumption means that the ego vehicle must create a new plan before it finishes executing its previously planned trajectory.
To simplify notation, we reset time to $0$ whenever a feasible control policy is about to be applied, \emph{i.e.}, $t_0+\tplan = 0$.
We denote the planning horizon by $T := [0,\tf]$.

In each planning iteration, \methodname{} chooses a \emph{desired trajectory} to be followed by the ego vehicle.
The desired trajectory is chosen from a pre-specified continuum of trajectories, with each uniquely determined by an $n_p$-dimensional \textit{trajectory parameter} $p \in \P\subset \R^{n_p}$.
We adapt the definition of trajectory parametrization from \cite[Definition 7]{liu2022refine}, and note three important details about the parameterization:
First, all desired trajectories share a time instant $\tnb$ that divides the desired trajectory into a \emph{driving maneuver} during $[\tz+\tplan,\tz+\tplan+\tm)$ and a \emph{contingency braking maneuver} during $[\tz+\tplan+\tm,\tz+\tplan+\tf]$.
Second, the contingency braking maneuver brings the ego vehicle to a stop by $\tf$.
This latter property ensures safety as we describe in Section \ref{sec:offline_reachability}.
Finally, a feedback controller is used to track this parameterized trajectory. 
Note an example of feedback controller parameterized by $p$ that satisfies this assumption can be found in \cite[Section V]{liu2022refine}. 

\subsection{Ego Vehicle Occupancy} 
\label{subsec: environment}
To provide guarantees about vehicle behavior in a receding horizon planning framework, we define the ego vehicle's footprint similar to \cite[Definition 10]{liu2022refine} as:
\begin{defn}
\label{def: footprint}
The ego vehicle is a rigid body that lies in a rectangle $\mathcal O^{ego}:=\Int([-0.5l,-0.5w]^T,[0.5l,0.5w]^T )\subset\W$ with width $w>0$ and length $l>0$ at time $t=0$. 
$\mathcal O^{ego}$ is called the \emph{footprint} of the ego vehicle.
\end{defn}
\noindent For arbitrary time $t$, given state $z(t)$ of the ego vehicle that starts from initial condition $z_0\in\Z_0\subset\Rnzt$ and applies a control input parameterized by $p\in\P$, the ego vehicle's \emph{forward occupancy} at time $t$ can be represented as
\begin{equation}
\label{eq: def occupancy}
    \hspace{-0.2cm}\E\big(t,z_0,p\big) := \rot(\heading)\cdot\Oego + [\sx(t),\sy(t)]^\top,
\end{equation}
which is a zonotope by \eqref{eq: zono-matrix mult}. 
Constructing such a representation of the forward occupancy is difficult because the vehicle dynamics are nonlinear.
Note throughout the remainder of the paper, we assume that $\Z_0$ is also a zonotope.

\subsection{Environment and Sensing}
We define obstacles as follows:
\begin{defn}
An \emph{obstacle} is a set $\Obskt\subset \W$ that the ego vehicle should not collide with at time $t$, where $k\in\K$ is the index of the obstacle and $\K$ contains finitely many elements. 
\end{defn}
\noindent The dependency on $t$ in the definition of an obstacle allows the obstacle to move as $t$ varies. 
Assuming that the ego vehicle has a maximum speed $\vegomax$ and all obstacles have a maximum speed $\vobsmax$ for all time, we make the following assumption on planning and sensing horizon. 
\begin{assum}
\label{ass: sense horizon}
The ego vehicle senses all obstacles within a sensor radius $\delta > (\tf +\tplan)\cdot(\vegomax+\vobsmax)+0.5\sqrt{l^2+w^2}$ around its center of mass.
\end{assum}
 \noindent This assumption ensures that any obstacle that may cause a collision between times $t\in[\tz+\tplan, \tz+\tplan+\tf]$ can be detected by the vehicle \cite[Thm. 15]{vaskov2019towards}.
Note one could treat sensor occlusions as obstacles that travel at the maximum obstacle speed \cite[]{yu2019occlusion,yu2020risk}.

\section{Formulating the Motion Planning Problem}\label{sec:overview}
Avoiding collisions in dynamic environments may not always be possible (e.g., a parked car can be run into).
Instead, we develop a trajectory synthesis technique which ensures that the ego vehicle is not-at-fault \cite[]{shalev2017formal}:
\begin{defn}\label{defn:notatfault}
The ego vehicle is \emph{not-at-fault} if it is stopped, or if it is never in collision with any obstacles while it is moving. 
\end{defn}
\noindent In other words, the ego vehicle is not responsible for a collision if it has stopped and another vehicle collides with it.
Note that in some instances coming to a stop may be unsafe. 
Instead one could use alternate definitions of not-at-fault that requires that the ego vehicle leaves enough time for all surrounding vehicles to come safely to a stop whenever it comes to a stop as discussed in \cite[]{liu2022refine}, or require that the final position of the ego vehicle be inside an invariant set as discussed in \cite[]{althoff2018efficient}. 
Note that the definition of not-at-fault could be generalized in this instance and the work presented in this paper could also be extended to accommodate this definition.
However, for simplicity, in this work we use the notion of not-at-fault behavior detailed in Def. \ref{defn:notatfault}.
This section formulates the not-at-fault motion planning problem using optimization. 

\subsection{Theoretical Motion Planning Problem}


Given the predicted initial condition of the vehicle $z_0$ at $t = 0$, one could attempt to compute a not-at-fault trajectory by solving the following theoretical optimization problem at each planning iteration: 
\begin{align*}
    \min_{p \in \P} & \quad \cost(z_0,p) \hspace{4cm} \thopt \label{eq:thoptcost}\\\\
    \text{s.t.}
    & \quad \E\big(t,z_0,p\big) \cap \Obskt =\emptyset, \hspace{0.5cm} \forall t\in T, \forall k\in\K
\end{align*}
where $\cost:\Z_0\times\P \to \R$ is a user-specified cost function.
Note that the constraint in $\thopt$ is satisfied if for a particular trajectory parameter $p$, the intersection between any obstacle $\Obskt$ and the forward occupancy of the ego vehicle while following $p$ is trivial for all $t \in T$.
Unfortunately implementing an optimization problem to solve this problem is challenging because it requires a representation of the forward occupancy of the vehicle for all $t \in T$.


\subsection{Offline Reachability Analysis}
\label{sec:offline_reachability}

To aid in the representation of the forward occupancy of the vehicle, we first partition the planning horizon $T := [0,\tf]$ into $\tf/\Delta_t$ \emph{time intervals} with some positive number $\Delta_t$ that divides $\tf$, and denote $\T_j$ the $j$-th time interval $[(j-1)\Delta_t, j\Delta_t]$ for any $j\in \J:=\{1,2,\ldots,\tf/\Delta_t\}$.
Next, as is described in \cite[]{liu2022refine} and \cite[]{liu2023radius}, we over-approximate the ego vehicle's trajectory using zonotopes as stated below:
\begin{assum}
\label{ass: offline reachability}
    Let $z$ be a solution to \eqref{eq:dynamics} starting from initial condition $z_0\in\Z_0$ with control parameter $p\in\PP$.
    For each $j\in\J$, there exists a map $\xi_j:\Z_0\times\PP\rightarrow P(\W)$ such that 
    \begin{enumerate}
        \item $\xi_j(z_0,p)$ contains the ego vehicle's footprint during $\T_j$, i.e., $\cup_{t\in \T_j} \E(t,z_0,p)\subseteq \xi_j(z_0,p)$, and
        \item $\xi_j(z_0,p)$ is a zonotope of the form $\zonocg{c_j(z_0)+A_j\cdot p}{G_j}$ with some linear function $c_j:\Z_0\rightarrow \R^2$, some matrix $A_j\in\R^{2\times n_p}$ and some 2-row matrix $G_j$.
    \end{enumerate}
    For convenience, we denote $\{\xi_j(z_0,p)\}_{j \in \J}$ by $\xi(z_0,p)$ and refer to it as the \emph{zonotope reachable set}.
\end{assum}

\noindent The zonotope reachable set can be constructed by applying offline reachability analysis techniques\cite[Section VI]{liu2022refine}.
Specifically, it can be generated by first applying the open-source toolbox CORA \cite[]{althoff2015introduction} to over-approximate the trajectory of the ego vehicle's state with initial condition $z_0$ and control parameter $p$ using a collection of zonotopes.
Each zonotope over-approximates a segment of the trajectory during each time interval $\T_j$ for $j\in\J$, and then accounts for the ego vehicle's footprint.
Explicit formulas for $\xi_j$, $c_j$, $A_j$ and $G_j$ can be found in the proof of Lem. 26 in \cite[]{liu2022refine}.
Note the formulas provided in \cite[Lem. 26]{liu2022refine} assume that $\xi_j$ is computed in the body frame of the ego vehicle, i.e., assuming $x(0)=y(0)=h(0)=0$.
In the case when the initial position and heading of the ego vehicle are not zeros, one can represent $\xi_j(z_0,p)$ in the world frame via a coordinate transformation based on $z_0$.
The remainder of the manuscript assumes $\xi_j(z_0,p)$ is represented in the world frame. 

Because prediction is not the primary emphasis of this work, we assume that the future position of any sensed obstacle within the sensor horizon during $[\tz, \tz+\tplan+\tf]$ is conservatively known at time $t_0$ and overapproximated using a zonotope:
\begin{assum}
\label{ass: obs in T}
    There exists a map $\vartheta:\J\times\K\rightarrow P(\W)$ such that $\vartheta(j,k)$ is a zonotope and
    \begin{equation}
        \cup_{t\in \T_j}\Obskt \cap B\left( (x(0),y(0)), \delta \right) \subseteq \vartheta(j,k).
    \end{equation}
    For convenience, we denote $\{\vartheta(j,k)\}_{j \in \J}$ by $\vartheta(k)$ and refer to it as the \emph{zonotope obstacle representation}.
\end{assum} 


\subsection{Online Optimization}
\label{sec:online_optimization}

Before formulating the optimization problem that \methodname{} solves at run-time, we have one final definition:
\begin{defn}\label{def:rdf_point}
Given initial condition $z_0\in\Z_0$ and control parameter $p\in\PP$, the \defemph{Reachability-based Distance Function}, or RDF, of the zonotope reachable set, $\xi(z_0,p)$, to the zonotope obstacle representation, $\vartheta(k)$, is defined as
\begin{equation}
    \rdf(\xi(z_0,p), \vartheta(k)) = \min_{j \in \J} \sdf( \xi_j(z_0,p), \vartheta(j,k) ).
\end{equation}
\end{defn}
\noindent Because $ \xi_j(z_0,p)$ and $\vartheta(j,k)$ are both zonotopes, one can compute the RDF using a convex quadratic program as described in Section \ref{subsec:def_sdf}.
By applying Assumptions \ref{ass: offline reachability} and \ref{ass: obs in T}, one can prove the following theorem:
\begin{thm} \label{thm:exact_rdf}
    Suppose the ego vehicle is starting from initial condition $z_0\in\Z_0$ with control parameter $p\in\PP$.
    Let $\xi$ and $\vartheta$ be the zonotope reachable set and zonotope obstacle representation as in Assumptions \ref{ass: offline reachability} and \ref{ass: obs in T}, respectively. 
    Then $\rdf(\xi(z_0,p), \vartheta(k)) \leq \rdf( \{\E\big(t,z_0,p\big)\}_{t\in T}, \{\Obskt\}_{t \in T})$. 
\end{thm}


Given the predicted initial condition of the vehicle at $t = 0$ as $z_0$, \methodname{} computes a not-at-fault trajectory by solving the following optimization problem at each planning iteration:
\begin{align}
    \min_{p \in \P} & \quad \cost(z_0,p) \hspace{4cm} \opt \label{eq:optcost}\\
    \text{s.t.}
    & \quad \rdf(\xi(z_0,p), \vartheta(k)) \geq 0, \hspace{0.5cm} \forall k\in\K.
\end{align}
As a result of Thm. \ref{thm:exact_rdf}, if a trajectory parameter satisfies the constraint in \optref, then it can be followed in a not-at-fault manner by the vehicle.  
Unfortunately \optref requires evaluating the constraint, which requires solving another optimization problem, which make solving \optref computationally prohibitive.

Instead of computing the distance between an obstacle and ego vehicle, REFINE enforces collision avoidance by requiring that a polytope corresponding to the reachable set does not intersect with a polytope corresponding to the obstacle. 
Each of these polytopes is represented as half-spaces. 
Unfortunately, as we illustrate in Section \ref{sec:results}, this constraint representation does not provide as useful of a descent direction as a constraint representation that measures distance by using the 2-norm. 
As a result \methodname{}'s signed distance based constraint 
is able to converge to a high-quality solution in far fewer steps than REFINE.

\section{Exact RDF with ReLU Networks}
\label{sec:nn}
This section describes how to compute the signed distance function between two zonotopes, how to represent this function exactly using a ReLU neural network, and how we can replace the obstacle-avoidance constraints in \optref with a novel neural implicit representation that overapproximates the distance between the forward occupancy of the ego vehicle and obstacles in the environment.

\subsection{Zonotope Vertex Enumeration}
Theorem \ref{thm:relu_sdf_network}, the main result of this section, relies on the enumeration of vertices of 2D zonotopes. 
In this subsection, we provide an algorithm and theorem that constructs a zonotope's vertices from its generators.

\begin{algorithm}[t]
    \caption{Zonotope Vertex Enumeration}
    \label{alg:vertex}
    \begin{algorithmic}[1]
        \REQUIRE Zonotope $Z \subset \mathbb{R}^2$ with generators $g_1, \ldots, g_m$
        \STATE Multiply any generator with a negative y-value by $-1$.
        \STATE Sort $g_1, \ldots, g_m$ by angle in ascending order. 
        \STATE Compute $C \in \mathbb{R}^{m \times m}$:
        \begin{equation*}
            C(i,j) = 
            \begin{cases}
                1 & \text{if } j \geq i \\
                -1 & \text{otherwise}
            \end{cases}.
        \end{equation*}
        \STATE  Compute $2m$ vertices $\{v_1, \cdots, v_{2m}\}$:
        \begin{align*}
            v_k &= c + \sum_{j=1}^{m} C(k,j) g_j \\
            v_{m+k} &= c - \sum_{j=1}^{m} C(k,j) g_j. 
       \end{align*}
        \STATE \textbf{End}
    \end{algorithmic}
\end{algorithm}

\begin{thm}
\label{thm:relu_sdf_network}
Let $\Z = \zonocg{c}{G} \subset \mathbb{R}^2$ be a zonotope with $m$ generators $g_1, \ldots, g_m$. 
Assume 
no two generators are scaled versions of one another. 
Then Alg. \ref{alg:vertex} produces the vertices of $\Z$. 
\end{thm}
\begin{proof} 
Let $C$ be the coefficient matrix and $V = \{v_1, \ldots, v_{2m}\}$ be the set of vertices computed by Alg. \ref{alg:vertex}.
To prove this theorem, it is necessary to show that $Z = \conv(V)$ and that the set of points in $V$ are vertices of $Z$.

We start by proving $Z = \conv(V)$. 
Let $p \in \conv(V)$ be arbitrary. 
Then there exist non-negative $\sigma_1, \ldots, \sigma_{2m}$, with $\sum_{i=1}^{2m} \sigma_j = 1$, such that $p = \sum_{i=1}^{2m} \sigma_i v_i$.
The point $p$ can be rewritten as
\begin{align}
    p &= \sum_{i=1}^{m} \sigma_i (c + \sum_{j=1}^{m} C(i,j) g_j) + \\ &\hspace*{1cm} + \sum_{i=1}^{m} \sigma_{m+i} (c - \sum_{j=1}^{m} C(i,j) g_j) \nonumber  \\
    &= c + \sum_{j=1}^{m}  \underbrace{ \Bigl( \sum_{i=1}^{m} (\sigma_i - \sigma_{m+i}) C(i,j) \Bigr) }_{\beta_j} g_j  \\
    &= c + \sum_{j=1}^{m} \beta_j g_j.
\end{align}
Because each $\beta_j \in [-1, 1]$, this implies $p \in \Z$ and $\conv(V) \subset \Z$.

Now let $p \in \Z$ be arbitrary. 
Then $p= c + \sum_{j=1}^{m} \beta_j g_j$ such that each $\beta_j \in [-1,1]$.
If we prove that there exists non-negative $\sigma_1, \ldots, \sigma_{2m}$, with $\sum_{i=1}^{2m} \sigma_j = 1$, such that
\begin{equation}
    \beta_j = \sum_{i=1}^{m} (\sigma_i - \sigma_{m+i}) C(i,j),
\end{equation}
then one can use the previous argument to prove that $p$ belongs to $\conv(V)$.
For $i=1,\ldots, 2m$, rewrite $\sigma_i = \sigma_i^{+} - \sigma_i^{-}$ with the additional constraint $\sigma_i^{-} = 0$.
This leads to the following system of equations:
\begin{equation}
    \label{eq:indeterminant_conv}
    \begin{bmatrix}
            \mathbf{\beta} \\
            1 \\
            \mathbf{0}_{2m \times 1}
    \end{bmatrix}
    =\begin{bmatrix}
            C & -C & \mathbf{0}_{m \times m} & \mathbf{0}_{m \times m}  \\
            \mathbf{1}_{1\times m} &  \mathbf{1}_{1\times m} & \mathbf{0}_{1\times m} & \mathbf{0}_{1\times m} \\
            \mathbf{0}_{m \times m}&  \mathbf{0}_{m \times m} & I_{m \times m} & \mathbf{0}_{m \times m} \\
            \mathbf{0}_{m \times m} &  \mathbf{0}_{m \times m} & \mathbf{0}_{m \times m} & I_{m \times m} 
    \end{bmatrix}
    \begin{bmatrix}
            \sigma^{+} \\
            \sigma^{-}
    \end{bmatrix}.
\end{equation}
Because the above matrix is rank deficient with linearly independent rows due to the definition of $C$, infinitely many solutions exist.
Therefore $p \in \conv(V)$ and $\Z \subset \conv(V)$.
Thus $\Z = \conv(V)$.

Next, to prove that $V$ are the vertices of $\Z$, one can apply proof by contradiction to show that no point in $V$ is a convex combination of the other points.

First note that the matrix $C$ is invertible.
This can be trivially seen adding the first row of $C$ to all rows to obtain the full rank upper triangular matrix: 
\begin{equation}
    \begin{bmatrix}
            2 & 2 & \dots & 2 & 2   \\
            0 &  2 & \dots & 2 & 2 \\
            \vdots &  \vdots & \ddots & \vdots & \vdots \\
            0 &  0 & \dots & 2 & 2 \\
            0 &  0 & \dots & 0 & 2
    \end{bmatrix}
    .
\end{equation}
WLOG, suppose there exists $k \leq m$  such that $v_k \in V$ is a convex combination of the other points of $V$.
That is, there exists non-negative $\sigma_1,\dots, \sigma_{k-1}, \sigma_{k+1}, \dots, \sigma_{2m} \in \mathbb{R}$ such that $\sum\limits_{\substack{n=1 \\ n \neq k}}^{2m} \sigma_n = 1$ and
\begin{align}
    v_k &= c + \sum_{j=1}^{m} C(k,j) g_j \nonumber \\
    &=  \sigma_1 \biggl( c + \sum_{j=1}^{m} C(1,j) g_j     \biggr) + \ldots + \nonumber \\ 
    &\hspace*{0.5cm} + \sigma_{2m} \biggl( c + \sum_{j=1}^{m} C(m,j) g_j     \biggr) \label{convexcomb}.
\end{align}
Rearranging the terms in \eqref{convexcomb} and recalling that $\sum\limits_{\substack{n=1 \\ n \neq k}}^{2m} \sigma_n = 1$ gives
\begin{multline}
    \sum_{j=1}^{m} (\sigma_{m+k} + 1) C(k,j) g_j = 
     \sum_{j=1}^{m}\biggl(
     \sum\limits_{\substack{n=1 \\ n \neq k}}^{m} (\sigma_n + \\ - \sigma_{m+n}) C(n,j) \biggr) g_j. \label{equatecoeffs}
\end{multline}
Equating the coefficients of \eqref{equatecoeffs} shows that the $k^{\text{th}}$ row of $C$ is a linear combination of the other rows of $C$.
This contradicts the fact that $C$ is invertible and therefore the set $V$ only containts vertices of $\Z$. $ \blacksquare $
\end{proof}

\subsection{Exact Signed Distance Between Zonotopes}
\label{subsec:rdf_derivation}


The following theorem leverages the properties of zonotopes to find a closed form expression for the exact signed distance between a collection of zonotopes:
\begin{thm} \label{thm:exact_zono}
Let $\Z = \zonocg{\cz}{\Gz}$ and $\Oi = \zonocg{\coi}{\Goi}$ be zonotopes in $\R^2$ for each $i \in \I$ where $\I$ has a finite number of elements in it.
Let $\Obs = \bigcup_{i \in \I} \Oi$.
Let $\Oiz = \zonocg{\coi}{[\Gz,\Goi]}$.
Suppose that for each $i \in \I$, $\Oiz$ has $N_i$ vertices.
Let $V_i$ be the set of vertices of $\Oiz$ ordered in counterclockwise fashion and $\{\segli\}_{l=1}^{N_i}$ be the set of line segments between consecutive vertices in $V_i$. 

The signed distance between $\Z$ and $\Obs$ is given by 
\begin{align}
   \hspace*{-0.48cm} \sdf(\Z, \Obs)
    &=
    \begin{cases}
          \underset{i \in \I}{\min}\: \underset{l \in \{1,\ldots,N_i\} }{\min}\: \dist(\cz; \segli)  & \text{if } \Z \cap \Obs =  \emptyset \\
          -\underset{i \in \I}{\min}\: \underset{l\in \{1,\ldots,N_i\}}{\min}\: \dist(\cz; \segli) & \text{if } \Z \cap \Obs  \neq \emptyset 
    \end{cases}.  \label{eq:rdf}
\end{align}
\end{thm}
\begin{proof} 
Assume without loss of generality that $\Z \cap \Oi =  \emptyset$. 
Note that $\sdf(\Z,\Oi) = \dist(\cz; \partial \Oiz)$ \cite[Theorem 19]{michaux2023rdf}.
Because $\partial \Oiz = \{ \segli \}_{l=1}^{N_i}$, the result follows by taking the minimum over $i \in \I$.
The proof for the case $Z \cap \Obs \neq \emptyset$ is similar.
Fig. \ref{fig:zono_sdf} illustrates this proof when the distance is being computed between two zonotopes. $ \blacksquare $
\end{proof}

Because of Ass. \ref{ass: offline reachability} and \ref{ass: obs in T}, one can apply Thm. \ref{thm:exact_zono} to compute the constraint in \optref.
However, because \eqref{eq:rdf} still requires computing the distance between a vertex and a finite length line segment, it naively would still require solving a convex quadratic program. 
The next section introduces a closed form method to compute this distance.

\begin{figure}[t]
    \centering
    \includegraphics[width=0.8\columnwidth]{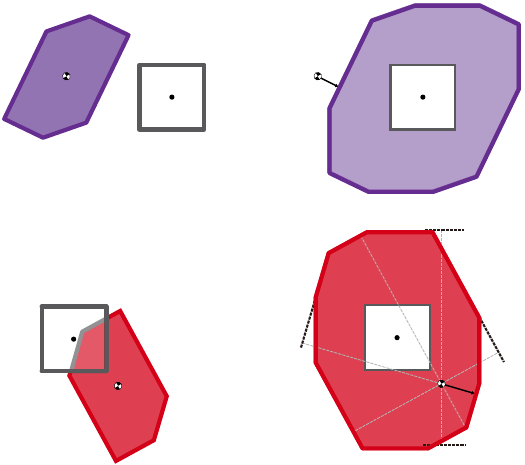}
    \caption{Graphical illustration of signed distance between zonotopes. 
     The top row illustrates the positive distance function between two zonotopes.
     First, the gray zonotope is buffered by the generators of the purple zonotope. 
     Then, the distance is computed between the purple zonotope's center and the closest line segment of the buffered zonotope. 
     The bottom row graphically illustrates the negative distance between two zonotopes. 
     First, the pink zonotope's generators buffer the gray zonotope. 
     Then, the negative distance is computed by projecting the pink zonotope's centers on each of the buffered zonotope's half-spaces.} 
    \label{fig:zono_sdf}
    \vspace*{-0.2cm}
\end{figure}

\subsection{Neural Network Architecture}
\label{subsec:nn_architecture}
We now use results from Sec. \ref{subsec:rdf_derivation} to enable the closed form expression of the signed distance function between zonotopes.
In particular, we show that the signed distance between collections of zonotopes can be computed exactly with an implicit neural network with fixed, pre-specified weights.
Before constructing this network, we note the following observation about max and min operations which is proven in \cite[Lem. D.3]{arora2016}, \cite[Sec. 3.2]{ferlez2020}:

\begin{lem}\label{lem:minmax}
The maximum or minimum of a pair of numbers can be computed using a ReLU network with width of $4$ and a depth of $2$ as:
\begin{align}
\min(x,y) &= \relu \left(   \begin{bmatrix}
            1 & 1 \\
            1 & -1 \\
            -1 & 1 \\
            -1 & -1
    \end{bmatrix} \begin{bmatrix} x \\ y \end{bmatrix} \right)^T  \begin{bmatrix}
        \frac{1}{2} \\ -\frac{1}{2} \\ -\frac{1}{2} \\ -\frac{1}{2}
    \end{bmatrix} \\ 
\max(x,y) &= \relu \left(  \begin{bmatrix}
            1 & 1 \\
            -1 & -1 \\
            -1 & 1 \\
            1 & -1
    \end{bmatrix} \begin{bmatrix} x \\ y \end{bmatrix} \right)^T\begin{bmatrix}
        \frac{1}{2} \\ -\frac{1}{2} \\ \frac{1}{2} \\ \frac{1}{2}
    \end{bmatrix}.
\end{align}
\end{lem}

By applying Lem. \ref{lem:minmax} recursively, one can implement multi-element min(max) operations with ReLU networks and construct the following bound:
\begin{cor}
\label{cor:relu_minmax_size}
The minimum (maximum) of $N$ numbers can be represented exactly with a ReLU network with width and depth bounded by $4 \cdot \bigl \lceil \frac{N}{2} \bigr \rceil$ and $2 \cdot \lfloor \log_{2} N \rfloor $, respectively.
\end{cor}
One can apply these results to prove the following theorem:
\begin{thm}
\label{thm:relu_sdf_network}
Let $\Z = \zonocg{\cz}{\Gz}$ and $\Oi = \zonocg{\coi}{\Goi}$ be zonotopes in $\R^2$ for each $i \in \I$ where $\I$ has a finite number of elements in it.
Let $\Obs = \bigcup_{i \in \I} \Oi$.
Let $\Oiz = \zonocg{\coi}{[\Gz,\Goi]}$.
Suppose that for each $i \in \I$, $\Oiz$ has $N_i$ vertices.
Then the signed distance between $\Z$ and $\Obs$ can be represented \textbf{exactly} using a ReLU network
width and depth bounded by  $4 \cdot \bigl \lceil \sum_{i \in \I} \frac{\sum_{i \in \I} N_i }{2} \bigr \rceil + 8$ and $2 \cdot \lfloor \log_{2}  ( \sum_{i \in \I} N_i ) \rfloor + 4$, respectively.
\end{thm}
\begin{proof} 
Let $V_i$ be the set of vertices of $\Oiz$ ordered in counterclockwise fashion and $\{\segli\}_{l=1}^{N_i}$ be the set of line segments between consecutive vertices in $V_i$.
Using Thm \ref{thm:exact_zono}, one can compute the signed distance between $\Z$ and $\Obs$ by applying Lem. \ref{lem:minmax}, if one can compute $\dist(\cz; \segli)$.

Let $v_{l,i}$ and $v_{l+1,i}$ be the vertices corresponding to line segment $\segli$.   
Let $H_{l,i}(\tau) = v_{l,i} + (v_{l+1,i} - v_{l,i})\tau$ for $\tau \in \mathbb{R}$.
Note, $H_{l,i}$ is the line that extends $\segli$.
The $\tau$ corresponding to the projection of $\cz$ onto $H_{l,i}$ is
\begin{equation}
    \hat{t}_{l,i} = \frac{(\cz - v_{l,i})^T(v_{l+1,i} - v_{l,i})}{(v_{l+1,i} - v_{l,i})^T (v_{l+1,i} - v_{l,i})}.
\end{equation}
Note that $H_{l,i}(\hat{t}_{l,i})$ is the point on $H_{l,i}$ closest to $\cz$, but it may not be on the line segment $\segli$.
Therefore, it is necessary to project $\hat{t}_{l,i}$ onto the interval $[0,1]$:
\begin{equation}\label{eq:segment_projection}
     t_{l,i}^{\ast} = \min (\max (\hat{t}_{l,i}, 0), 1).
\end{equation}
One can compute \eqref{eq:segment_projection} using Lem. \ref{lem:minmax} by performing $2$ ReLU operations, which requires a ReLU network of width and depth bounded by $8$ and $4$, respectively. 
Then the distance between $\cz$ and $\segli$ is 
\begin{equation}\label{eq:segment_distance}
    \dist(\cz; \segli) = \| \cz - (v_l + (v_{l+1} - v_l)t_{l,i}^{\ast})\|_{2}.
\end{equation}

The result follows by noting that to compute the signed distance between $\Z$ and $\Obs$ using \eqref{eq:rdf}, one needs to perform \eqref{eq:segment_projection} once for each line segment in $\Oiz$ and noting that $\Obs$ has $\sum_{i\in \I} N_i$ line segments. $ \blacksquare $
\end{proof}

\subsection{Trajectory Optimization}
\label{sec:nn_trajopt}

One can use this distance function to formulate the following trajectory optimization problem:
\begin{align}
    \min_{p \in \P} & \quad \cost(z_0,p) \hspace{2.95cm} \nnopt \label{eq:nnoptcost}\\
    \text{s.t.}
    & \quad \ardf_{NN}(\FOzj\big(z_0,p\big), \Obszjk) > 0, \hspace{0.5cm} \forall j\in\J, \forall k\in\K,
\end{align}
where $j$ indexes the time interval,
$\FOzj\big(z_0,p\big)$ is the zonotope forward reachable set, $\Obszjk$ is the $k$\ts{th} obstacle zonotope at the $j$\ts{th} time interval, and $\ardf_{NN}(\FOzj\big(z_0,p\big), \Obszjk)$ is the neural network representation described in Sec. \ref{subsec:nn_architecture} and illustrated in Fig. \ref{fig:overview}.
Note this neural network is constructed by applying Thm. \ref{thm:relu_sdf_network} and its gradient can be constructed by applying backpropagation. 

\subsection{Online Operation}

Algorithm \ref{alg:methodname} summarizes the online operation of \methodname{}.
Each planning iteration, the ego vehicle executes the feasible control parameter that was computed in the previous planning iteration (Line 3).
At the beginning of each planning iteration, \senseobs{} senses obstacles and predicts their future motion (Line 4) in local frame decided by $\zpos_0$.
\nnoptref is then solved to compute a control parameter $p^*$ using $z_{0}$ (Line 5).
\onlineopt{} includes zonotope vertex enumeration (Alg. \ref{alg:vertex}), constructing the neural signed distance function (Thm. \ref{thm:relu_sdf_network}), and solving \nnoptref.
If \nnoptref does not find a feasible solution within $\tplan$, \methodname{} executes a contingency braking maneuver whose safety was verified in the previous planning iteration, and stops planning (Line 6).
If \nnoptref is able to find a feasible $p^*$, \statepred{} predicts the state value at $t=\tnb$ based on $z_{0}$ and $p^*$ as in Ass. \ref{assum:tplan} (Lines 7 and 8).
\begin{algorithm}[t]
    \caption{\methodname{} Online Planning}
    \label{alg:methodname}
    \begin{algorithmic}[1]
        \REQUIRE $p_0\in\P$ and $z_0 = [(\zpos_0)^\top,(\zvel_0)^\top]^\top\in\R^3\times\Zvel_0$
        \STATE \textbf{Initialize:} $p^*\gets p_0$, $t\gets 0$
        \STATE \textbf{Loop:} // {\it Line 3 executes concurrently with Lines 4-8}
            \STATE \quad \textbf{Execute} $p^*$ during $[0, \tnb)$
            \STATE \quad  $\{\vartheta(j,i,\zpos_0)\}_{(j,i)\in\J\times\I}\gets\texttt{SenseObstacles}()$ \label{eq:senseobs}
            \STATE \quad \textbf{Try} $p^*\gets\texttt{OnlineOpt}(z_0,\{\vartheta(j,i,\zpos_0)\}_{(j,i)\in\J\times\I})$ \\ \quad // {\it within $\tplan$ seconds} \label{eq:onlineopt}
            \STATE \quad  \textbf{Catch} execute failsafe maneuver, then \textbf{break}
            \STATE \quad  $(\zpos_0,\zvel_0)\gets\texttt{StatePrediction}(z_0,p^*,\tnb)$ \label{eq:stateprediction}
            \STATE \quad $z_0\gets [(\zpos_0)^\top,(\zvel_0)^\top]^\top$
            \STATE \quad \textbf{If} ($\zvel_0\notin\Zvel_0$), execute failsafe maneuver and \textbf{break}
            \STATE \quad  \textbf{Reset} $t$ to 0
        \STATE \textbf{End}
    \end{algorithmic}
\end{algorithm}

If the predicted velocity value does not belong to $\Zvel_0$, then its corresponding FRS is not available.
So the ego vehicle executes the braking maneuver and \methodname{} stops planning (Line 9).
Otherwise, time is reset to 0 (Line 10) and the next planning iteration begins.
Note Lines 4 and 7 are assumed to execute instantaneously, but in practice the time spent for these steps can be subtracted from $\tplan$ to ensure real-time performance.

By iteratively applying Def. \ref{defn:notatfault}, Ass. \ref{ass: offline reachability} and Ass. \ref{ass: obs in T}, the following theorem holds:
\begin{thm}
Suppose the ego vehicle can sense and predict surrounding obstacles as in Assumption \ref{ass: obs in T}, and starts with a not-at-fault trajectory parameter $p_0\in\P$. 
Then by performing planning and control as in Algorithm \ref{alg:methodname}, the ego vehicle is not-at-fault for all time.
\end{thm}

\section{Experimental Setup}
\label{sec:experimental_setup}
This section describes our experimental setup including simulation environments and how the optimization problem is implemented. 

\subsection{Implementation Details}

The experiments for \methodname{} and REFINE \cite[]{liu2022refine} were conducted on a computer with an AMD EPYC 7742 64-Core CPU and a NVIDIA RTX A6000 GPU. 
The \methodname{} network \ref{subsec:nn_architecture} is built and compiled with Pytorch \cite[]{paszke2017automatic}. 
Additionally, \methodname{} utilizes IPOPT \cite[]{ipopt-cite} for online planning by solving  \nnoptref in Sec. \ref{sec:nn_trajopt}.
In Table \ref{table: simulation result}, we compare this work's implementations of \methodname{} and REFINE with the results originally reported in \cite[Table 1]{liu2022refine}.
Note in all implementations, analytical gradients are provided and IPOPT uses a quasi-newton method.

\subsection{Simulation and Simulation Environment}
\label{subsec: simulation}
We evaluate the performance of \methodname{} on 1000 randomly generated, highway driving scenarios similar to \cite[]{liu2022refine}. 
In this simulation, each method must execute multiple planning iterations in a receding horizon fashion to successfully navigate through dynamic traffic to the end of a 1000[m] long highway.
The highway has 3 lanes that are 3.7[m] wide.
Each scenario also contains up to 15 moving vehicles and up to 3 static vehicles randomly placed across the highway as obstacles.

\subsection{Neural Network Implementation}
The neural network is a computation graph composed of min/max $\relu$ networks that computes the minimum distance from a point to a collection of line segments (Fig. \ref{fig:overview}).
Note that the depth and width \ref{thm:relu_sdf_network} of the neural network change during each planning iteration based on the collection of zonotopes comprising the current ego vehicle's and obstacle vehicles' reachable sets.
In practice each min/max $\relu$ network is implemented with PyTorch's $\texttt{clamp()}$ and $\texttt{min/max()}$ for simplicity.

\subsection{Optimization Problem Implementation}

Both REFINE and \methodname{} pre-compute a library of desired trajectories and an associated set of control parameters $\P$, where each trajectory has a corresponding control parameter in $\P$.
These are used to generate a collection of FRSs, where each FRS in this collection corresponds to a partition of the trajectory parameter space (i.e., $\cup_{b = 1}^{M} \P_b = \P$ where the $\P_b$'s are disjoint). 
Note that this is done to ensure that each FRS is a tight over-approximation to the behavior of the vehicle.
We refer to each FRS in this partition, and its accompanying set of control parameters $\P_b$ as a bin.

During online planning, REFINE solves an optimization problem to find a control parameter $p_b\in\P_b$ for each bin in parallel to solve the larger optimization for the entire collection of FRSs. 


To formulate the cost function in \optref, a high level planner (HLP) generates a waypoint at each planning iteration. 
The HLP first chooses the lane where the obstacles in that lane are furthest away from the current position of the ego vehicle $z_0$.
Next it outputs a position waypoint in that lane that is some pre-specified minimum distance away from the location of the nearest obstacle in that lane.
Finally, based on the ego vehicle's initial state $z_0$ and the decision variable (control parameter) $p$, we compute the predicted ego vehicle location and corresponding cost function in \optref as its Euclidean distance to the specified waypoint.
Note this HLP is given to all methods including REFINE and \methodname{}.

REFINE then evaluates the cost of the optimization solution of each bin and chooses the solution with the minimum cost as the control parameter and trajectory to be executed.
In this process, REFINE represents its obstacle avoidance constraints as polytope half-space constraints to ensure that the reachable set zonotopes do not intersect with the obstacle zonotopes.
As we illustrate in Section \ref{sec:results}, using this representation of the obstacle avoidance constraint requires taking many more steps during trajectory synthesis when compared to \methodname{}.

The neural network representation of the safety constraint proposed in this paper can allow for rapid batch constraint evaluations from multiple bins.
As a result, rather than solve multiple optimization problems in parallel, it is more computationally efficient for \methodname{} to batch all of the single-bin optimizations into a single large optimization problem.
To perform this single large optimization problem within \methodname{}, it needs to identify which bins have a feasible solution. 
To do this, we sample control parameters within each bin at 10 equally spaced point in the range of the control parameter.
When solving the batched optimization problem, we only keep the bins $\{b_i\} \subseteq \sB$ which are feasible at any of the selected sample points.
Note that the time of sampling is also taken into account when we record the overall algorithm runtime.

Recall \optref and let subscript $b_i$ denote the single-bin optimization index we run during online planning.
We then modify the original optimization in \optref by batching all single bin optimizations into one as:
\begin{align}
    \min_{\{p_{b_i}\}_{i=1}^{M} \subseteq \P} & \quad \sum_{i=1}^M \cost(z_{0,b_i},p_{b_i}) \hspace{0.5cm} \bopt \label{eq:boptcost}\\
    \text{s.t.}
    & \quad \ardf_{NN}(\FOzj\phantom{}_{,b_i}\big(z_{0,b_i},p_{b_i}\big), \Obszjk) > 0 \nonumber \\
    & \hspace{3.0cm} \forall j\in\J, \forall k\in\K, \forall b_i\in \sB
\end{align}
where $\FOzj\phantom{}_{,b}\big(z_{0,b},p_i\big)$ is the zonotope forward reachable set from bin $b$ and $\Obszjk$ is the $k$\ts{th} obstacle zonotope at the $j$\ts{th} time step. $\Tilde{r}_{NN}$ has the neural network representation described in Section \ref{subsec:nn_architecture}.

\section{Results}
\label{sec:results}
This section evaluates the performance of \methodname{} for online trajectory planning in simulation. 
We compare \methodname{}'s success rate and computation speed to a variety of state-of-the-art methods. 
We also compare \methodname{}'s constraint evaluation to REFINE in terms of evaluation time and the number of times it is called during each optimization step.
Additionally, we evaluate how \methodname{} compares to REFINE in a number of highway driving scenarios when both methods are only given a limited time to generate a motion plan to navigate through these environments.

\subsection{Constraint Evaluation Comparison}
\label{subsec: constraint_eval_comp}

In both REFINE and \methodname{} most of the time spent during the optimizations is in evaluating the constraints and their gradients.
To investigate how \methodname{}'s constraint representation and optimization formulation, we analyze the speed of the evaluation of the constraint and its gradient for both algorithms, as well as how many times IPOPT evaluates these quantities in each motion planning iteration.

\subsubsection{Constraint Evaluation Speed Comparison}

We first examine both methods' constraint and constraint gradient evaluation run time across all bins. 
To do this, we input 1000 different values of the optimization variable into the constraint and constraint gradient functions and measure their run time on the same initial vehicle and obstacle configuration. 
 We note that both methods run on the full range of bins in this evaluation (i.e., bins in which there may or may not be trajectory parameter that satisfies all of the constraints).
 During each planning iteration, REFINE launches three rounds of optimizations with $13$ single-bin optimization problems being solved in parallel each round. 
 REFINE's constraint and constraint gradient runtime is then measured as the average time of a constraint evaluation and a constraint gradient evaluation on each single-bin optimization in the three rounds of parallel optimization.
  \methodname{}, on the other hand, evaluates the constraint and constraint gradient for all bins in a batched fashion.
 Its constraint and  constraint gradient evaluation run time is then measured as the time taken for a forward/backward pass of the described distance neural network.

Table \ref{table: constraint time} records the average runtime of constraint and constraint gradient evaluations for \methodname{} and REFINE. 
Even though \methodname{} is faster in solving the overall optimization problem, REFINE's constraint and constraint gradient evaluations are faster than \methodname{}'s.
This difference in speed is likely due to different constraint formulations.
Each of REFINE's constraints is represented as a half-space constraint, whereas each of \methodname{}'s constraints involves additional computations for calculating the signed distance. 
Despite this increase in complexity of \methodname{}'s constraints, we show in Sec. \ref{subsubsec: no_constraint_evals} that the signed distance representation reduces the number of constraint evaluations during planning.

\begin{table}[t]
    \centering
    \begin{tabular}{|c||c|c|}
    \hline 
    Method & \makecell{Mean constraint \\ evaluation  \\ time [ms]} & \makecell{Mean constraint \\ gradient evaluation \\ time [ms]} \\ \hline
    REFINE & \textbf{2.17 ± 0.58} & \textbf{2.33 ± 0.60} \\ \hline
    \methodname{} & 4.50 ± 0.16 & 4.54 ± 0.17\\ \hline
    \end{tabular}
    \caption{Mean constraint evaluation time of REFINE and \methodname{}. }
    \label{table: constraint time}
\end{table}

\subsubsection{Number of Constraint Evaluations Comparison}
\label{subsubsec: no_constraint_evals}
Because \methodname{} and REFINE use different constraint representations, their constraints also provide different gradients.
We compare how many times constraint and constraint gradient evaluations are invoked within each optimization problem for \methodname{} and REFINE. 

Table \ref{table: num constraints} summarizes the average number of constraint and constraint gradient evaluations invoked during an online optimization problem. 
Note that we only allow each optimization problem to run a limited number of iterations; within each iteration, the constraint and constraint gradient can be evaluated multiple times by IPOPT \cite[]{ipopt-cite}. 
We observe that \methodname{} on average requires less than half of the constraint evaluations for each optimization problem than REFINE under the same number of maximum IPOPT optimization steps allowed.
This pattern also holds when we allow \methodname{} to run more IPOPT optimization steps than REFINE.
We also note that \methodname{} combines all bins to a single batched optimization problem, while REFINE launches three rounds of parallel processes to solve one optimization problem each bin. 
In light of this, the \# constraint evaluations and \# constraint gradient evaluations for REFINE reported in Tab.~\ref{table: num constraints} are the sum of the average number of constraint evaluations for each of the three rounds of parallel processes REFINE runs.



\begin{table}[t]
    \centering
    \begin{tabular}{|c||c|c|}
    \hline 
    Method & \makecell{\# constraint \\ evaluations} &  \makecell{\# constraint gradient \\ evaluations} \\ \hline
    \makecell{REFINE\\(max\_iter=15)} & 104.6 ± 123.0 & 40.2 ± 5.5 \\ \hline
    \makecell{\methodname{}\\(max\_iter=15)} & \textbf{15.4 ± 5.1} & \textbf{14.8 ± 2.8}\\ \hline
    \makecell{\methodname{}\\(max\_iter=20)} & 17.6 ± 8.3 & 16.3 ± 4.4\\ \hline
    \makecell{\methodname{}\\(max\_iter=25)} & 19.8 ± 12.6 & 17.5 ± 6.1\\ \hline
    \end{tabular}
    \caption{Mean total number of constraint evaluations of REFINE and \methodname{} within a single optimization problem. We use max\_iter to specify the maximum allowed number of optimization steps that IPOPT can take.
    } 
    \label{table: num constraints}
\end{table}





\begin{figure}[t]
    \includegraphics[width=0.99\columnwidth]{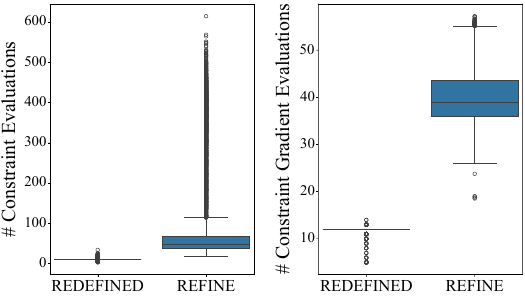}
    \caption{Mean number of constraint and constraint gradient evaluations of \methodname{} vs. REFINE within one optimization problem.}
    \label{fig:box_plot}
    \vspace*{-0.25cm}
\end{figure}

\subsection{Motion Planning}
\label{subsec: motion planning}

\begin{table*}[t]
    \centering
    \begin{tabular}{|c||c|c|c|c|c|c|}
    \hline 

    \multirow{2}{*}{Method} & \multirow{2}{*}{Success} & \multirow{2}{*}{Crash } & \multirow{2}{*}{\makecell{Safely Stop}} & \makecell{Solve Time  of Online Planning $\downarrow$} \\ 
    & &  & & (Average, Maximum)\\\hline
    \makecell{Baseline (sparse) \\ \cite[]{manzinger2020using}} & 62\% & {\bf 0\%} & 38\%  &  (2.03[s], 4.15[s]) \\ \hline
    \makecell{Baseline (dense) \\ \cite[]{manzinger2020using} } & 70\% & {\bf 0\%} & 30\%  &  (12.42[s], 27.74[s]) \\ \hline
    \makecell{SOS-RTD \\ \cite[]{kousik2017safe}} & 64\% & {\bf 0\%} & 36\% &  ({\bf 0.05[s]}, 1.58[s]) \\ \hline
    \makecell{NMPC \\ \cite[]{falcone2008low}}& 68\% & 29\% & 3\% &   (40.89[s], 534.82[s])  \\ \hline
    \makecell{REFINE \\ \cite[]{liu2022refine}} & \textbf{82\%} & {\bf 0\%}& 18\% &  (0.50[s], 2.30[s])\\ \hline
    \methodname{} & \textbf{82\%} & {\bf 0\%}& 18\% &  (0.29[s], \textbf{1.46[s]})\\ \hline
    \end{tabular}
    \caption{Summary of performance of various tested techniques on the same $1000$ simulation environments.}
    \label{table: simulation result}
\end{table*}
\subsubsection{Comparison algorithms}

\begin{figure*}
    \centering
    
    \begin{subfigure}[b]{0.99\textwidth}
      \includegraphics{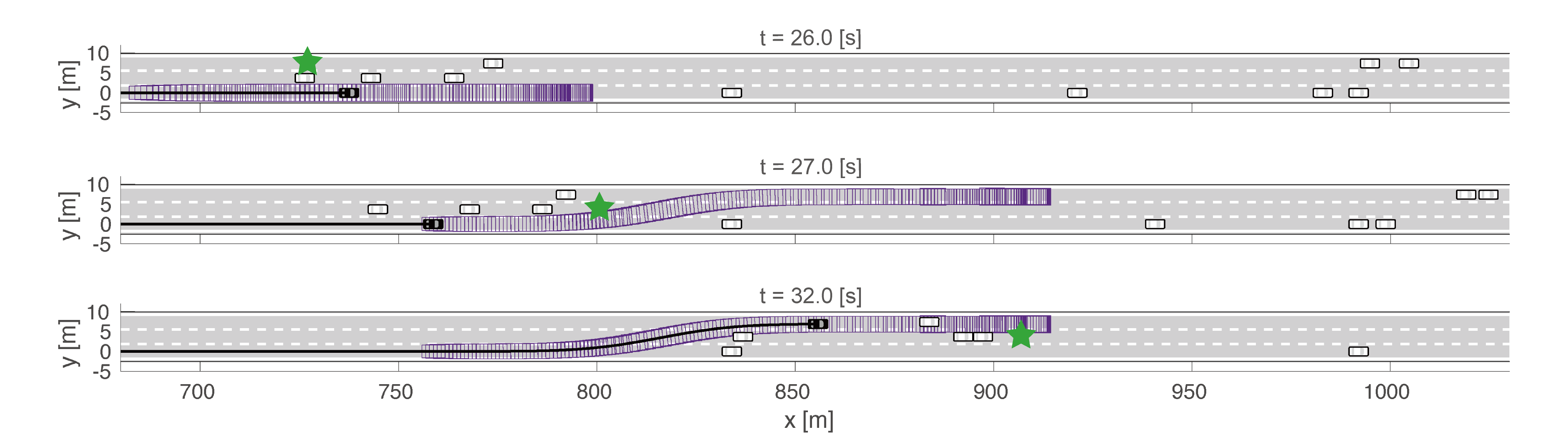}
      \caption{\methodname{} utilized.}
      \label{fig:redefined}
    \end{subfigure}

    \begin{subfigure}[b]{0.99\textwidth}
        \includegraphics{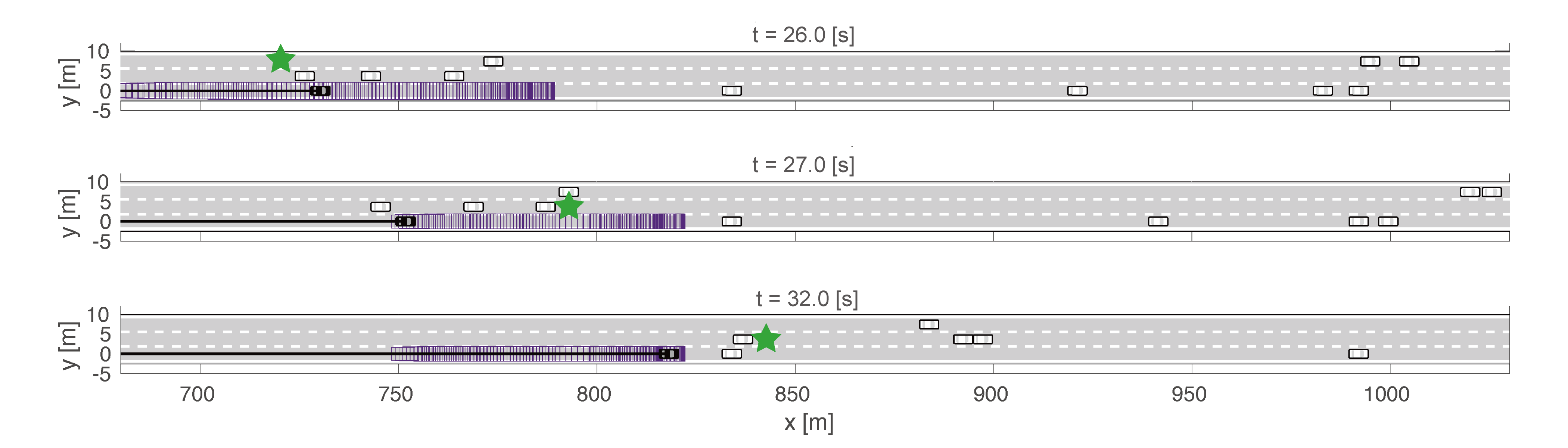}
        \caption{REFINE utilized.}
        \label{fig:refine}
    \end{subfigure}

    \begin{subfigure}[b]{0.99\textwidth}
        \includegraphics{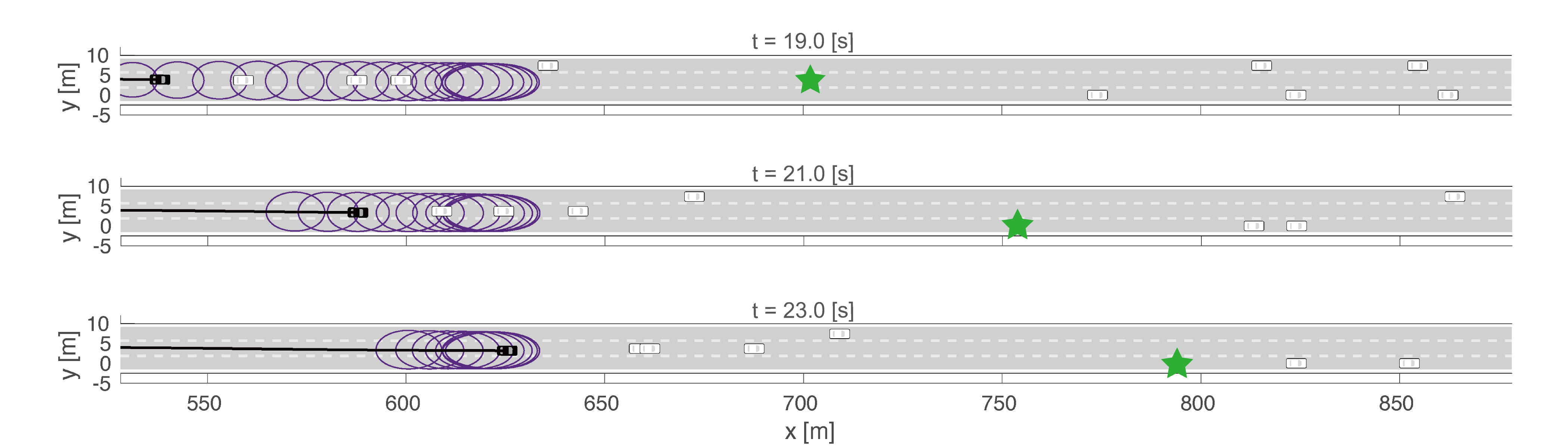}
        \caption{SOS-RTD utilized.}
        \label{fig:sos}
    \end{subfigure}
    
    \begin{subfigure}[b]{0.99\textwidth}
        \includegraphics{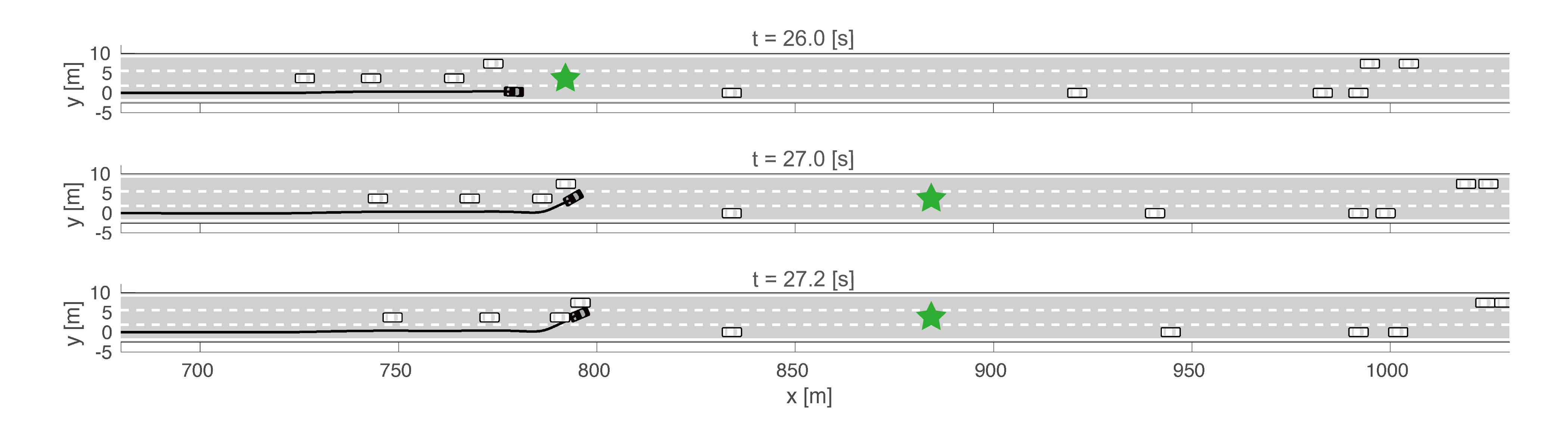}
        \caption{NMPC utilized.}
        \label{fig:nmpc}
    \end{subfigure}

    \caption{An illustration of \methodname{}, REFINE, SOS-RTD, and NMPC running on the same simulated scenario. (a) \methodname{} is able to find a solution at $t = 27.0$[s] to perform a lane change and finally succeeds. (b) REFINE fails to find a solution at $t = 27.0$[s] and therefore executes fail-safe maneuver to stop the ego vehicle. (c) SOS-RTD stops the ego vehicle at $t = 23.0$[s] because of its conservative reachable sets. (d) NMPC leads the ego vehicle to collision at $t = 27.2$[s]. In each set of images, the ego vehicle and its executed trajectory are colored in black. Zonotope reachable sets for \methodname{}, REFINE and polynomial reachable sets for SOS-RTD are colored in purple. Other vehicles are obstacles and are depicted in white. Waypoints are plotted as green stars.  }
    \label{fig:fig-motion-planning}
\end{figure*}

We compare \methodname{} to REFINE \cite[]{liu2022refine} and a collection state-of-the-art safe motion planning algorithms: 
another zonotope reachable set approach as detailed in \cite[]{manzinger2020using}, a trajectory planning method that leverages  Sum-of-Squares programming (called SOS-RTD) as mentioned in \cite[]{kousik2020bridging}, and a Nonlinear Model Predictive Control (NMPC) strategy utilizing GPOPS-II \cite[]{patterson2014gpops}.

The first trajectory planning strategy we deploy is a baseline zonotope reachability method. 
This method opts for a finite number of potential trajectories instead of an extensive range of options, which is the approach REFINE and \methodname{} employ. 
Similar to the well-established funnel library technique used in motion planning referenced in \cite[]{majumdar2017funnel}, this basic strategy selects a limited collection of potential paths for following. 
To calculate zonotope reachable sets, the baseline method applies the CORA \cite[]{Althoff2015a} framework.
Since this method uses only a discrete set of trajectories, we generate two versions of this baseline method, once with a sparse control parameter space and another with a dense space and use the same set of discrete parameter highlighted in the experimental section of \cite[]{liu2022refine}.
At runtime this baseline method searches through the discrete parameter space and select a trajectory parameter whose zonotope reachable set does not intersect with any obstacles.

SOS-RTD uses a control strategy that tracks families of trajectories, to execute maneuvers like speed changes, direction changes, and lane changes, all with an appended braking maneuver.
SOS-RTD approximates the FRS offline by solving a sequence of polynomial optimization problems utilizing Sum-of-Squares programming. 
This allows the FRS to be over-approximated as the union of superlevel sets of polynomials for consecutive time segments, each lasting 0.1 seconds \cite[]{kousik2019safe}. 
The calculated polynomial FRSs are then expanded by the footprints of other vehicles.
During the real-time optimization process, SOS-RTD performs planning at three-second intervals and utilizes the identical cost function use in \methodname{}. 

The NMPC method does not employ any reachability analysis.
Instead it covers the footprint of the ego vehicle and obstacle vehicles by two overlapping balls and directly solves an optimal control problem to generate a collision free trajectory.
This optimal control problem is implemented using GPOPS-II and solves for the set of control input that generate this collision free trajectory.
Note that this method does not append any braking maneuvers to the generated trajectory so does not have any fail-safes should it be unable to find a solution.

Lastly, REFINE utilizes feedback-linearization based control strategy to track a set of parameterized trajectories to execute maneuvers like speed changes, direction changes, and lane changes, all with an appended braking maneuver.
To compute the zonotope reachable sets, the REFINE also uses CORA.
With the parameterized trajectories and corresponding reachable sets REFINE solves an optimization problem over a continuous control parameter space to find a control parameter that results in a collision-free trajectory that minimizes the cost.
Note that REFINE uses non-intersection constraints between each zonotope and obstacle footprint that it recasts into a half-plane representation to pass into the optimization problem.
With larger number of obstacles this results in a large number of constraints that must be evaluated.

\subsubsection{Results}
We run \methodname{} against state-of-the-art trajectory planning methods implemented in \cite[]{liu2022refine} on the goal reaching task described in Sec. \ref{subsec: simulation}. 
Each method is run on the same set of 1000 trials of the highway simulation to test its ability to navigate to the goal located at 1000[m] away from its starting location. 
Each method's safe-stop rate, crash rate, success rate, and solving time of online planning problem are reported in Table \ref{table: simulation result}.
Note that Table \ref{table: simulation result} is an extension of Table I in \cite[]{liu2022refine} with \methodname{} being evaluated on the same scenarios as the other methods.

We observe that \methodname{} achieves same success rate while being approximately 1.7 times faster than REFINE. 
Compared to \methodname{}, the baseline methods with both sparse and dense control parameter discretization result in a lower success rate with longer solve time. 
SOS-RTD plans faster than \methodname{} on average, but achieves success less frequently because its polynomial reachable sets are more conservative. 
Similarly, NMPC has a lower success rate and crashes 29\% of the trials since it does not include reachability analysis or fail-safe maneuvers in its planned trajectory. 

Figure \ref{fig:fig-motion-planning} demonstrates four of the methods applied on the same trial.
In Figure \ref{fig:redefined}, \methodname{} is able to find a feasible solution at 27.0[s] and executes a lane change to traverse through the traffic. 
In Figure \ref{fig:refine}, REFINE fails to find a solution at 27.0[s], so it safely stops by executing the fail-safe part of the trajectory.
In Figure \ref{fig:sos}, SOS-RTD invokes a fail-safe maneuver at an earlier stage of t=23.0[s] because its polynomial reachable sets are too conservative. 
In Figure \ref{fig:nmpc}, NMPC drives aggressively to reach the waypoint at 27.0[s] and ends up too close to the obstacles, leading to a collision at 27.2[s].

\subsection{Single Planning Iteration} 

\subsubsection{Limited Planning Time}
\label{single_limited_plantime}
 Because we expect the methods to perform trajectory planning in highly dynamic environments in \textbf{\textit{real-time}}, we evaluate \methodname{} and REFINE while limiting their planning time. 
We first examine \methodname{} and REFINE's ability to solve the trajectory optimization problem in a single planning iteration.
In this experiment, we randomly generate 500 scenarios each with 10, 20, 30, 40 and 50 obstacles within 200 meters ahead of the ego vehicle, leading to a denser arrangement of obstacles compared to Section \ref{subsec: motion planning}.
We run \methodname{} and REFINE on the same set of scenarios for a single planning iteration while limiting their planning time to 0.35 seconds and compare their solution quality.

Because both methods run on the same problem with the same cost function, we can compare their solution quality by comparing the cost achieved at their final iteration.
Recall that each planning iteration involves solving optimization problems on multiple bins.
Therefore, we consider the minimum cost generated among the bins and take their average across the planning iterations.
The results are recorded in Tab. \ref{table: limited time optimization}.
 Notably \methodname{} consistently solves the optimization problems with lower costs in all cases.

\begin{table*}[t]
    \centering
    \resizebox{1.0\textwidth}{!}{
    \begin{tabular}{|c||c|c|c|c|c|}
    \hline 
    \multirow{2}{*}{Method} & \multicolumn{5}{c|}{Number of Obstacles} \\ \cline{2-6}
     & 10 &  20 & 30 & 40 & 50 \\ \hline
    \methodname{} & \textbf{10.8 ± 10.9 (11.1)} & \textbf{12.8 ± 12.7 (11.1)} & \textbf{13.4 ± 14.0 (11.1)} & \textbf{14.2 ± 12.9 (11.1)} & \textbf{15.1 ± 15.2 (11.1)} \\ \hline
    REFINE        & 11.6 ± 10.5 (11.2) & 13.8 ± 12.9 (11.3) & 14.6 ± 14.3 (11.3) & 15.7 ± 13.2 (11.5) & 16.1 ± 14.6 (11.5) \\ \hline
    \end{tabular}
    }
    \caption{Averaged minimum costs from 500 planning iterations of \methodname{} and REFINE solving optimization problems with 0.35-second time limit and different numbers of obstacles.
    Note the format is mean ± std (median). } 
    \label{table: limited time optimization}
\end{table*}

\subsubsection{Unlimited Planning Time} 
To investigate how \methodname{}'s computation times scaled with increasing number of obstacles, we repeated the same experiment as illustrated in Section \ref{single_limited_plantime} with no time limit.

The results detailed in Tab. \ref{table: single time} show that \methodname{}'s run time exhibits a linear increase with increasing number of obstacles, while REFINE's run time grows in a quadratic fashion. 
We report the median run time where a feasible solution of the problem is found to mitigate the influence of outliers caused by infeasible problems.
We attribute this enhanced scaling ability of \methodname{} compared to REFINE to the ability to easily batch the constraint evaluations for multiple obstacles due to the neural network formulation, as well as the more informative gradient, which allows the optimization to find better solutions.
These results indicate how \methodname{}'s more informative gradient and neural network representation allows it to scale better with increasing number of obstacles while  consistently providing better quality solutions with a limited planning time.

\begin{table}[t]
    \centering
    \begin{tabular}{|c||c|c|c|c|c|}
    \hline 
    \multirow{2}{*}{Method} & \multicolumn{5}{c|}{Number of Obstacles} \\ \cline{2-6}
     & 10 &  20 & 30 & 40 & 50 \\ \hline
    \methodname{} & \textbf{0.24} & \textbf{0.32} & \textbf{0.38} & \textbf{0.43} & \textbf{0.46} \\ \hline
    REFINE & 0.49 & 0.73 & 1.02 & 1.59 & 4.65 \\ \hline
    \end{tabular}
    \caption{Median runtime (unit: [s]) from 500 planning iterations of \methodname{} and REFINE solving optimization problems with  various numbers of obstacles.} 
    \label{table: single time}
\end{table}

\subsection{Receding Horizon Planning Iteration With Hard Time Limit}

We next compared \methodname{} and REFINE on their performance under limited time in the receding horizon planning framework described in \ref{subsec: motion planning}.
We run \methodname{} and REFINE on 1000 trials each with a time limit of 0.35, 0.3, and 0.25 seconds. 
We generate the scenarios such that each scenario has 30 dynamic obstacles located within 400m ahead of the ego vehicle.
No static obstacles are included.
We report the trial success rate (i.e., how often the ego vehicle reaches the goal) in Table \ref{table: limited time planning}. No crashes happen in all cases.

We observe that \methodname{} is able to consistently reach the goal more frequently than REFINE in these experiments.
The more informative gradient, coupled with the capability to batch the optimization over all obstacles in parallel, allowed \methodname{} to be able to solve the optimization more effectively and much faster that REFINE which yielded a higher success rate.

\begin{table}[t]
    \centering
    \begin{tabular}{|c||c|c|c|}
    \hline 
    \multirow{2}{*}{Method} & \multicolumn{3}{c|}{Time Limit} \\ \cline{2-4}
     & 0.35[s] &  0.3[s] & 0.25[s] \\ \hline
    \methodname{} & \textbf{84.8} & \textbf{84.3} & \textbf{76.0} \\ \hline
    REFINE        & 81.1 & 81.7 & 75.0 \\ \hline
    \end{tabular}
    \caption{Trial success rate from 1000 trials of \methodname{} and REFINE running receding horizon planning with 0.35, 0.3, and 0.25-second time limit.} 
    \label{table: limited time planning}
\end{table}

\section{Conclusion}
This paper describes an exact neural implicit representation for the signed distance function between zonotopes and its use as a safety constraint for autonomous vehicle motion planning.
Our novel planning method, REDEFINED, combines reachability analysis and an exact neural signed distance function to enable real-time vehicle motion planning with safety guarantees and demonstrates advantages in solution quality, scalability, and run time when applied in trajectory optimization.
We note that the current limitations of our method include the use of zonotopes to over-approximate all the obstacles and the necessity of online computation of the vertices of the zonotopes that may not scale well to planning in higher dimensional spaces (e.g., 3D space).
The results presented here will be useful for combining rigorous model-based methods with flexible deep learning approaches for safe robot motion planning.




 \printbibliography

@article{schafer2021computation,
  title={Computation of solution spaces for optimization-based trajectory planning},
  author={Sch{\"a}fer, Lukas and Manzinger, Stefanie and Althoff, Matthias},
  journal={IEEE Transactions on Intelligent Vehicles},
  year={2021},
  publisher={IEEE}
}

@article{manzinger2020using,
  title={Using reachable sets for trajectory planning of automated vehicles},
  author={Manzinger, Stefanie and Pek, Christian and Althoff, Matthias},
  journal={IEEE Transactions on Intelligent Vehicles},
  volume={6},
  number={2},
  pages={232--248},
  year={2020},
  publisher={IEEE}
}

@INPROCEEDINGS{althoff2018efficient,
  author={Pek, Christian and Althoff, Matthias},
  booktitle={2018 IEEE/RSJ International Conference on Intelligent Robots and Systems (IROS)}, 
  title={Efficient Computation of Invariably Safe States for Motion Planning of Self-Driving Vehicles}, 
  year={2018},
  volume={},
  number={},
  pages={3523-3530},
  doi={10.1109/IROS.2018.8593597}}

@article{arora2016,
  author       = {Raman Arora and
                  Amitabh Basu and
                  Poorya Mianjy and
                  Anirbit Mukherjee},
  title        = {Understanding Deep Neural Networks with Rectified Linear Units},
  journal      = {CoRR},
  volume       = {abs/1611.01491},
  year         = {2016},
  url          = {http://arxiv.org/abs/1611.01491},
  eprinttype    = {arXiv},
  eprint       = {1611.01491},
  bibsource    = {dblp computer science bibliography, https://dblp.org}
}

@inproceedings{ferlez2020,
author = {Ferlez, James and Shoukry, Yasser},
title = {AReN: Assured ReLU NN Architecture for Model Predictive Control of LTI Systems},
year = {2020},
isbn = {9781450370189},
publisher = {Association for Computing Machinery},
address = {New York, NY, USA},
url = {https://doi.org/10.1145/3365365.3382213},
doi = {10.1145/3365365.3382213},
booktitle = {Proceedings of the 23rd International Conference on Hybrid Systems: Computation and Control},
articleno = {6},
numpages = {11},
location = {Sydney, New South Wales, Australia},
series = {HSCC '20}
}

@misc{liu2022refine,
      title={REFINE: Reachability-based Trajectory Design using Robust Feedback Linearization and Zonotopes}, 
      author={Jinsun Liu and Yifei Shao and Lucas Lymburner and Hansen Qin and Vishrut Kaushik and Lena Trang and Ruiyang Wang and Vladimir Ivanovic and H. Eric Tseng and Ram Vasudevan},
      year={2022},
      eprint={2211.11997},
      archivePrefix={arXiv},
      primaryClass={cs.RO}
}

@misc{michaux2023rdf,
      title={Reachability-based Trajectory Design with Neural Implicit Safety Constraints}, 
      author={Jonathan Michaux and Qingyi Chen and Yongseok Kwon and Ram Vasudevan},
      year={2023},
      eprint={2302.07352},
      archivePrefix={arXiv},
      primaryClass={cs.RO}
}

@inproceedings{kousik2019safe,
  title={Safe, aggressive quadrotor flight via reachability-based trajectory design},
  author={Kousik, Shreyas and Holmes, Patrick and Vasudevan, Ram},
  booktitle={ASME 2019 Dynamic Systems and Control Conference},
  year={2019},
  organization={American Society of Mechanical Engineers Digital Collection}
}

@article{kousik2020bridging,
  title={Bridging the gap between safety and real-time performance in receding-horizon trajectory design for mobile robots},
  author={Kousik, Shreyas and Vaskov, Sean and Bu, Fan and Johnson-Roberson, Matthew and Vasudevan, Ram},
  journal={The International Journal of Robotics Research},
  volume={39},
  number={12},
  pages={1419--1469},
  year={2020},
  publisher={SAGE Publications Sage UK: London, England}
}

@article{vaskov2019towards,
  title={Towards provably not-at-fault control of autonomous robots in arbitrary dynamic environments},
  author={Vaskov, Sean and Kousik, Shreyas and Larson, Hannah and Bu, Fan and Ward, James and Worrall, Stewart and Johnson-Roberson, Matthew and Vasudevan, Ram},
  journal={arXiv preprint arXiv:1902.02851},
  year={2019}
}

@article{yu2019occlusion,
  title={Occlusion-aware risk assessment for autonomous driving in urban environments},
  author={Yu, Ming-Yuan and Vasudevan, Ram and Johnson-Roberson, Matthew},
  journal={IEEE Robotics and Automation Letters},
  volume={4},
  number={2},
  pages={2235--2241},
  year={2019},
  publisher={IEEE}
}

@inproceedings{yu2020risk,
  title={Risk assessment and planning with bidirectional reachability for autonomous driving},
  author={Yu, Ming-Yuan and Vasudevan, Ram and Johnson-Roberson, Matthew},
  booktitle={2020 IEEE International Conference on Robotics and Automation (ICRA)},
  pages={5363--5369},
  year={2020},
  organization={IEEE}
}

@inproceedings{kousik2017safe,
  title={Safe trajectory synthesis for autonomous driving in unforeseen environments},
  author={Kousik, Shreyas and Vaskov, Sean and Johnson-Roberson, Matthew and Vasudevan, Ram},
  booktitle={ASME 2017 Dynamic Systems and Control Conference},
  year={2017},
  organization={American Society of Mechanical Engineers Digital Collection}
}

@article{althoff2014online,
  title={Online verification of automated road vehicles using reachability analysis},
  author={Althoff, Matthias and Dolan, John M},
  journal={IEEE Transactions on Robotics},
  volume={30},
  number={4},
  pages={903--918},
  year={2014},
  publisher={IEEE}
}

@inproceedings{ Althoff2015a,
	author = {Althoff, M.},
	title = {An Introduction to CORA 2015},
	booktitle = {Proc. of the Workshop on Applied Verification for Continuous and Hybrid Systems},
	year = {2015},
}

@book{lavalle2006planning,
  title={Planning algorithms},
  author={LaValle, Steven M},
  year={2006},
  publisher={Cambridge university press}
}

@article{janson2015fast,
  title={Fast marching tree: A fast marching sampling-based method for optimal motion planning in many dimensions},
  author={Janson, Lucas and Schmerling, Edward and Clark, Ashley and Pavone, Marco},
  journal={The International journal of robotics research},
  volume={34},
  number={7},
  pages={883--921},
  year={2015},
  publisher={SAGE Publications Sage UK: London, England}
}

@article{elbanhawi2014sampling,
  title={Sampling-based robot motion planning: A review},
  author={Elbanhawi, Mohamed and Simic, Milan},
  journal={Ieee access},
  volume={2},
  pages={56--77},
  year={2014},
  publisher={IEEE}
}

@article{kuwata2009real,
  title={Real-time motion planning with applications to autonomous urban driving},
  author={Kuwata, Yoshiaki and Teo, Justin and Fiore, Gaston and Karaman, Sertac and Frazzoli, Emilio and How, Jonathan P},
  journal={IEEE Transactions on control systems technology},
  volume={17},
  number={5},
  pages={1105--1118},
  year={2009},
  publisher={IEEE}
}

@inproceedings{falcone2008low,
  title={Low complexity mpc schemes for integrated vehicle dynamics control problems},
  author={Falcone, Paolo and Borrelli, Francesco and Asgari, J and Tseng, HE and Hrovat, Davor},
  booktitle={9th international symposium on advanced vehicle control (AVEC)},
  year={2008}
}

@inproceedings{wurts2018collision,
  title={Collision imminent steering using nonlinear model predictive control},
  author={Wurts, John and Stein, Jeffrey L and Ersal, Tulga},
  booktitle={2018 Annual American Control Conference (ACC)},
  pages={4772--4777},
  year={2018},
  organization={IEEE}
}

@article{shalev2017formal,
  title={On a formal model of safe and scalable self-driving cars},
  author={Shalev-Shwartz, Shai and Shammah, Shaked and Shashua, Amnon},
  journal={arXiv preprint arXiv:1708.06374},
  year={2017}
}

@inproceedings{althoff2015introduction,
  title={An introduction to CORA 2015},
  author={Althoff, Matthias},
  booktitle={Proc. of the Workshop on Applied Verification for Continuous and Hybrid Systems},
  year={2015}
}

@article{majumdar2017funnel,
  title={Funnel libraries for real-time robust feedback motion planning},
  author={Majumdar, Anirudha and Tedrake, Russ},
  journal={The International Journal of Robotics Research},
  volume={36},
  number={8},
  pages={947--982},
  year={2017},
  publisher={SAGE Publications Sage UK: London, England}
}

@article{howard2007optimal,
  title={Optimal rough terrain trajectory generation for wheeled mobile robots},
  author={Howard, Thomas M and Kelly, Alonzo},
  journal={The International Journal of Robotics Research},
  volume={26},
  number={2},
  pages={141--166},
  year={2007},
  publisher={Sage Publications Sage CA: Thousand Oaks, CA}
}

@article{patterson2014gpops,
  title={GPOPS-II: A MATLAB software for solving multiple-phase optimal control problems using hp-adaptive Gaussian quadrature collocation methods and sparse nonlinear programming},
  author={Patterson, Michael A and Rao, Anil V},
  journal={ACM Transactions on Mathematical Software (TOMS)},
  volume={41},
  number={1},
  pages={1--37},
  year={2014},
  publisher={ACM New York, NY, USA}
}

@book{boyd2004convex,
  title={Convex optimization},
  author={Boyd, Stephen and Vandenberghe, Lieven},
  year={2004},
  publisher={Cambridge university press}
}

@article{paszke2017automatic,
  title={Automatic differentiation in PyTorch},
  author={Paszke, Adam and Gross, Sam and Chintala, Soumith and Chanan, Gregory and Yang, Edward and DeVito, Zachary and Lin, Zeming and Desmaison, Alban and Antiga, Luca and Lerer, Adam},
  year={2017}
}

@article{ipopt-cite,
author = {Wächter, Andreas and Biegler, Lorenz},
year = {2006},
month = {03},
pages = {25-57},
title = {On the Implementation of an Interior-Point Filter Line-Search Algorithm for Large-Scale Nonlinear Programming},
volume = {106},
journal = {Mathematical programming},
doi = {10.1007/s10107-004-0559-y}
}

@article{liu2023radius,
  title={RADIUS: Risk-Aware, Real-Time, Reachability-Based Motion Planning},
  author={Liu, Jinsun and Adu, Challen Enninful and Lymburner, Lucas and Kaushik, Vishrut and Trang, Lena and Vasudevan, Ram},
  journal={arXiv preprint arXiv:2302.07933},
  year={2023}
}
 
 \appendices

\end{document}